\pdfoutput=1 
\documentclass[11pt]{article} 
\usepackage{fancyhdr, natbib, fullpage}
\usepackage{parskip}
\usepackage{times}

\usepackage{hyperref}
\usepackage{url}

\usepackage{times}
\usepackage{amsmath,amsfonts,amsthm,amssymb}
\usepackage{setspace}
\usepackage{fancyhdr}
\usepackage{multirow}
\usepackage{makecell}
\usepackage{array}
\usepackage{soul,color}
\usepackage{graphicx,float,wrapfig}
\usepackage{ifpdf}
\usepackage{listings}
\usepackage{diagbox}
\usepackage{algorithm}
\usepackage[noend]{algorithmic}
\renewcommand{\algorithmiccomment}[1]{{\color{blue} \bgroup\hfill//~#1\egroup}}
\usepackage{enumitem}

\DeclareMathOperator*{\argmax}{arg\,max}

\newcommand{\E}{\mathbb{E}}
\newcommand{\R}{\mathbb{R}}

\newcommand{\ca}{\mathcal{A}}
\newcommand{\ta}{\tilde{a}}
\newcommand{\tx}{\tilde{x}}

\def\shownotes{1}  
\ifnum\shownotes=1
\newcommand{\authnote}[2]{{\scriptsize $\ll$\textsf{#1 notes: #2}$\gg$}}
\else
\newcommand{\authnote}[2]{}
\fi

\newcommand{\dingwen}[1]{{\color{blue}\authnote{Dingwen}{#1}}}
\newcommand{\yw}[1]{{\color{red}\authnote{Yuanhao}{#1}}}

\usepackage{amsfonts}
\usepackage{nicefrac}
\usepackage{microtype}      
\usepackage{xcolor}         
\usepackage{graphicx}
\usepackage{booktabs} 
\usepackage{booktabs}   
\usepackage{mathrsfs}
\usepackage{amsmath,amssymb}

\usepackage{graphicx,subcaption}

\usepackage{verbatim}
\usepackage{makecell}
\usepackage{microtype}  \usepackage{amsthm}
\usepackage{color}
\usepackage{hyperref,url}
\hypersetup{
  colorlinks,
  citecolor=blue!70!black,
  linkcolor=blue!70!black,
  urlcolor=blue!70!black}

\usepackage{bbm}

\newcommand{\cA}{\mathcal{A}}

\newcommand{\EE}{\mathbb{E}}

\newcommand{\one}{\mathbbm{1}}

\renewcommand{\*}{\star}
\renewcommand{\tilde}{\widetilde}

\renewcommand{\log}{\ln}

\usepackage{apptools}
\newtheorem{theorem}{Theorem}
\AtAppendix{\counterwithin{theorem}{section}}
\newtheorem{lemma}[theorem]{Lemma}
\newtheorem{conjecture}[theorem]{Conjecture}
\newtheorem{definition}{Definition}

\newtheorem{proposition}[theorem]{Proposition}
\newtheorem{remark}[theorem]{Remark}

\newtheorem{lem}[theorem]{Lemma}
\newtheorem{prop}[theorem]{Proposition}

\newtheorem{coro}[theorem]{Corollary}

\newenvironment{proof-sketch}{\noindent{\bf Proof Sketch}
  \hspace*{1em}}{\qed\bigskip\\}
\newenvironment{proof-idea}{\noindent{\bf Proof Idea}
  \hspace*{1em}}{\qed\bigskip\\}
\newenvironment{proof-of}[1][{}]{\noindent{\bf Proof of \cref{#1}}
  \hspace*{1em}}{\qed\bigskip\\}
\newenvironment{proof-of-lemma}[1][{}]{\noindent{\bf Proof of Lemma {#1}}
  \hspace*{1em}}{\qed\bigskip\\}
\newenvironment{proof-of-proposition}[1][{}]{\noindent{\bf
    Proof of Proposition {#1}}
  \hspace*{1em}}{\qed\bigskip\\}
\newenvironment{proof-of-theorem}[1][{}]{\noindent{\bf Proof of Theorem {#1}}
  \hspace*{1em}}{\qed\bigskip\\}
\newenvironment{inner-proof}{\noindent{\bf Proof}\hspace{1em}}{
  $\bigtriangledown$\medskip\\}
\newenvironment{proof-attempt}{\noindent{\bf Proof Attempt}
  \hspace*{1em}}{\qed\bigskip\\}

\title{Learning Rationalizable Equilibria in Multiplayer Games}

\author{Yuanhao Wang\thanks{Equal contribution.}${}^{\ast 1}$, Dingwen Kong\footnotemark[1]${}^{\ast 2}$, Yu Bai${}^{3}$, Chi Jin${}^{1}$ \\
${}^{1}$Princeton University, ${}^{2}$Peking University, ${}^{3}$Salesforce Research \\
\texttt{yuanhao@princeton.edu, dingwenk@pku.edu.cn}\\ \texttt{yu.bai@salesforce.com, chij@princeton.edu}
}

\begin{document}

\maketitle

\begin{abstract}

A natural goal in multiagent learning besides finding equilibria is to learn \emph{rationalizable} behavior, where players learn to avoid \emph{iteratively dominated actions}. 
However, even in the basic setting of multiplayer general-sum games, 
existing algorithms require a number of samples \emph{exponential} in the number of players to learn rationalizable equilibria under bandit feedback. 
This paper develops the first line of efficient algorithms for learning rationalizable Coarse Correlated Equilibria (CCE) and Correlated Equilibria (CE) whose sample complexities are \emph{polynomial} in all problem parameters including the number of players. To achieve this result, we also develop a new efficient algorithm for the simpler task of finding one rationalizable action profile (not necessarily an equilibrium), whose sample complexity  substantially improves over the best existing results of \citet{wu2021multi}. Our algorithms incorporate several novel techniques to guarantee rationalizability and no (swap-)regret simultaneously, including a correlated exploration scheme and adaptive learning rates, which may be of independent interest. We complement our results with a sample complexity lower bound showing the sharpness of our guarantees.




\end{abstract}
\section{Introduction}


A common objective in multiagent learning is to find various \emph{equilibria}, such as Nash equilibria (NE), correlated equilibria (CE) and coarse correlated equilibria (CCE). Generally speaking, a player in equilibrium lacks incentive to deviate assuming conformity of other players to the same equilibrium. Equilibrium learning has been extensively studied in the literature of game theory and online learning, and no-regret based learners can provably learn approximate CE and CCE with both computational and statistical efficiency~\citep{stoltz2005incomplete,cesa2006prediction}.

However, not all equilibria are created equal. As shown by~\citet{viossat2013no}, a CCE can be entirely supported on dominated actions---actions that are worse off than some other strategy in all circumstances---which rational agents should apparently never play. Approximate CE also suffers from a similar problem. As shown by~\citet[Theorem 1]{wu2021multi}, there are examples where an $\epsilon$-CE always plays iteratively dominated actions---actions that would be eliminated when iteratively deleting strictly dominated actions---unless $\epsilon$ is exponentially small. It is also shown that standard no-regret algorithms are indeed prone to finding such seemingly undesirable solutions~\citep{wu2021multi}.
The intrinsic reason behind this is that CCE and approximate CE may not be \emph{rationalizable}, and existing algorithms can indeed fail to find rationalizable solutions.

Different from equilibria notions, rationalizability~\citep{bernheim1984rationalizable,pearce1984rationalizable} looks at the game from the perspective of a single player without knowledge of the actual strategies of other players, and only assumes common knowledge of their rationality. A rationalizable strategy will avoid strictly dominated actions, and assuming other players have also eliminated their dominated actions, iteratively avoid strictly dominated actions in the subgame. Rationalizability is a central solution concept in game theory~\citep{osborne1994course} and has found applications in auctions~\citep{battigalli2003rationalizable} and mechanism design~\citep{bergemann2011rationalizable}.

If an (approximate) equilibrium only employs rationalizable actions, it would prevent irrational behavior such as playing dominated actions. Such equilibria are arguably more reasonable than unrationalizable ones, and constitute a stronger solution concept. This motivates us to consider the following open question:
\begin{center}
    \emph{Can we efficiently learn equilibria that are also rationalizable?}
\end{center}

Despite its fundamental role in multiagent reasoning, rationalizability is rarely studied from a learning perspective until recently, with~\citet{wu2021multi} giving the first algorithm for learning rationalizable action profiles from bandit feedback. However, these rationalizable action profiles found are not necessarily equilibria, and the problem of learning rationalizable CE and CCE remains a challenging open problem. Due to the existence of unrationalizable equilibria, running standard CE or CCE learners will not guarantee rationalizable solutions. On the other hand, one cannot hope to first identify all rationalizable actions and then find an equilibrium on the subgame, since even determining whether an action is rationalizable requires exponentially many samples (see Proposition~\ref{prop:decidehard}). Therefore, achieving rationalizability and approximate equilibria simultaneously is nontrivial and presents new algorithmic challenges.

\newcolumntype{C}{>{\centering\arraybackslash}p{4cm}}

\begin{table}[t]
\centering
  \renewcommand{\arraystretch}{1.5} 
\begin{tabular}{|cc|C|}
\hline
\multicolumn{2}{|l|}{\makecell{\textbf{Task}}}                                           & {\textbf{Sample Complexity}}  \\ \hline
\multicolumn{2}{|l|}{Find \emph{all} rationalizable actions (Proposition~\ref{prop:decidehard})}             & $\Omega(A^{N-1})$           \\ \hline
\multicolumn{2}{|l|}{Find \emph{one} rationalizable action profile (Theorem~\ref{thm:ide_upper})}              & $\tilde{O}\left(\frac{LNA}{\Delta^{2}}\right)$ \\ \hline
\multicolumn{1}{|l|}{\multirow{2}{*}{\makecell{Learn rationalizable\\ equilibria}}} & $\epsilon$-CCE (Theorem~\ref{thm:cce_main}) &   $\tilde{O}\left(\frac{LNA}{\Delta^{2}}+\frac{NA}{\epsilon^2}\right)$                        \\ \cline{2-3} 
\multicolumn{1}{|l|}{}                                                  & $\epsilon$-CE (Theorem~\ref{thm:ce_main}) & $\tilde{O}\left(\frac{LNA}{\Delta^{2}}+\frac{NA^2}{\min\{\epsilon^2,\Delta^2\}}\right)$                         \\ \hline
\end{tabular}
\caption{Summary of main results. Here $N$ is the number of players, $A$ is the number of actions per player, $L<NA$ is the minimum elimination length and $\Delta$ is the error we allow for rationalizability.  \label{table:summary}}

\end{table}

In this work, we address the challenges above and give a positive answer to our main question. Our contributions, summarized in Table~\ref{table:summary}, are the following:
\begin{itemize}[leftmargin=2em]
    \item As a first step, we provide a simple yet sample-efficient algorithm for identifying a $\Delta$-rationalizable action profile under bandit feedback, using only $\Tilde{O}\left(\frac{LNA}{\Delta^2}\right)$\footnote{Throughout this paper, we use $\tilde{O}$ to suppress logarithmic factors in $N$, $A$, $L$, $\frac{1}{\Delta}$, $\frac{1}{\delta}$, and $\frac{1}{\epsilon}$.} samples in normal-form games with $N$ players, $A$ actions per player and a minimum elimination length of $L$. This greatly improves the result of~\citet{wu2021multi} and is tight up to logarithmic factors when $L=O(1)$.
    \item Using the above algorithm as a subroutine, we develop exponential weights based algorithms that can provably find $\Delta$-rationalizable $\epsilon$-CCE using $\Tilde{O}\left(\frac{LNA}{\Delta^2}+\frac{NA}{\epsilon^2}\right)$ samples, and $\Delta$-rationalizable $\epsilon$-CE using $\Tilde{O}\left(\frac{LNA}{\Delta^2}+\frac{NA^2}{
    \min\{\epsilon^2,\Delta^2\}}\right)$ samples. To the best of our knowledge, these are the first guarantees for learning rationalizable approximate CCE and CE.
    \item We also provide reduction schemes that find $\Delta$-rationalizable $\epsilon$-CCE/CE using black-box algorithms for $\epsilon$-CCE/CE. Despite having slightly worse rates, these algorithms can directly leverage the progress in equilibria finding, which may be of independent interest.
\end{itemize}

\subsection{Related work}

\paragraph{Rationalizability and iterative dominance elimination.} Rationalizability~\citep{bernheim1984rationalizable,pearce1984rationalizable} is a notion that captures rational reasoning in games and relaxes Nash Equilibrium. 
Rationalizability is closely related to the iterative elimination of dominated actions, which has been a focus of game theory research since the 1950s~\citep{luce1957games}. It can be shown that an action is rationalizable if and only if it survives iterative elimination of strictly dominated actions\footnote{For this equivalence to hold, we need to allow dominance by mixed strategies, and correlated beliefs when there are more than two players. These conditions are met in the setting of this work.}~\citep{pearce1984rationalizable}. There is also experimental evidence supporting iterative elimination of dominated strategies as a model of human reasoning~\citep{camerer2011behavioral}.

\paragraph{Equilibria learning in games.} 

There is a rich literature on applying online learning algorithms to learning equilibria in games. It is well-known that if all agents have no-regret, the resulting empirical average would be an $\epsilon$-CCE~\citep{peyton2004strategic}, while if all agents have no swap-regret, the resulting empirical average would be an $\epsilon$-CE~\citep{hart2000simple, cesa2006prediction}. Later work continuing this line of research include those with faster convergence rates~\citep{syrgkanis2015fast,chen2020hedging,daskalakis2021near}, last-iterate convergence guarantees~\citep{daskalakis2018last, wei2020linear}, and extension to extensive-form games~\citep{celli2020no,bai2022near,bai2022efficient,song2022sample} and Markov games~\citep{song2021can,jin2021v}.

\paragraph{Computational and learning aspect of rationalizability.}

Despite its conceptual importance, rationalizability and iterative dominance elimination are not well studied from a computational or learning perspective. For iterative strict dominance elimination in two-player games, \citet{knuth1988note} provided a cubic-time algorithm and proved that the problem is P-complete. The weak dominance version of the problem is proven to be NP-complete by~\citet{conitzer2005complexity}.

\citet{hofbauer1996evolutionary} showed that in a class of learning dynamics which includes replicator dynamics --- the continuous-time variant of FTRL, all iteratively strictly dominated actions vanish over time, while~\citet{mertikopoulos2010emergence} proved similar results for stochastic replicator dynamics; however, neither work provides finite-time guarantees.~\citet{cohen2017hedging} proved that Hedge eliminates dominated actions in finite time, but did not extend their results to the more challenging case of iteratively dominated actions.

The most related work in literature is the work on learning rationalizable actions by~\citet{wu2021multi}, who proposed the Exp3-DH algorithm to find a strategy
mostly supported on rationalizable actions with a polynomial rate. Our Algorithm~\ref{alg:findaction} accomplishes the same task 
with a faster rate, while our Algorithms~\ref{alg:CCE-adaptive} \&~\ref{alg:CE-adaptive} deal with the more challenging problems of finding $\epsilon$-CE/CCE that are also rationalizable. Although Exp3-DH is based on a no-regret algorithm, it does not enjoy regret or weighted regret guarantees and thus does not provably find rationalizable equilibria.

\section{Preliminary}
An $N$-player normal-form game involves $N$ players whose action space are denoted by $\ca=\ca_1\times \cdots \times \ca_N$, and is defined by utility functions $u_1,\cdots,u_N:$ $\ca\to [0,1]$. Let $A=\max_{i\in [N]} |A_i|$ denote the maximum number of actions per player, $x_i$ denote a mixed strategy of the $i$-th player (\emph{i.e.}, a distribution over $\ca_i$) and $x_{-i}$ denote a (correlated) mixed strategy of the other players (\emph{i.e.}, a distribution over $\prod_{j\neq i}\ca_j$). We use $\Delta(S)$ to denote a distribution over the set $S$.
\paragraph{Learning from bandit feedback} We consider the bandit feedback setting where in each round, each player $i\in[N]$ chooses an action $a_i\in\ca_i$, and then observes a random feedback $U_i\in[0,1]$ such that
\[
\EE[U_i|a_1,a_2,\cdots,a_n]=u_i(a_1,a_2,\cdots,a_n).
\]


\subsection{Rationalizability}
An action $a\in \ca_i$ is said to be rationalizable if it could be the best response to some (possibly correlated) belief of other players' strategies, assuming that they are also rational. In other words, the set of rationalizable actions is obtained by iteratively removing actions that could never be a best response. For finite normal-form games, this is in fact equivalent to the iterative elimination of strictly dominated actions\footnote{See, \emph{e.g.}, the Diamond-In-the-Rough (DIR) games in~\citet[Definition 2]{wu2021multi} for a concrete example of iterative dominance elimination.}~\citep[Lemma 60.1]{osborne1994course}.


\begin{definition}[$\Delta$-Rationalizability]
\label{def:ration}
Define
\begin{equation*}
    E_1:=\bigcup_{i=1}^N\left\{
a\in\ca_i: \exists x\in \Delta({\ca_i}), \forall a_{-i},\quad  u_i(a,a_{-i}) \le u_i(x,a_{-i}) - \Delta
\right\},
\end{equation*}
which is the set of $\Delta$-dominated actions for all players. Further define
\begin{equation*}
    E_l:=\bigcup_{i=1}^N\left\{ a\in\ca_i: \exists x\in \Delta({\ca_i}), \forall a_{-i} \text{ s.t. } a_{-i}\cap E_{l-1}=\emptyset,  u_i(a,a_{-i}) \le u_i(x,a_{-i}) - \Delta\right\},
\end{equation*}
which is the set of actions that would be eliminated by the $l$-th round. Define  $L=\inf\{l: E_{l+1}=E_l\}$ as the \textbf{minimum elimination length}, and $E_L$ as the set of $\Delta$-\textbf{iteratively dominated actions} ($\Delta$-IDAs). Actions in $\cup_{i=1}^N \ca_i \setminus E_L$ are said to be \textbf{$\Delta$-rationalizable}. 
\end{definition}

Notice that $E_1\subseteq \cdots \subseteq E_L=E_{L+1}$. Here $\Delta$ plays a similar role as the reward gap for best arm identification in stochastic multi-armed bandits.
We will henceforth use $\Delta$-rationalizability and survival of $L$ rounds of iterative dominance elimination (IDE) interchangeably\footnote{Alternatively one can also define $\Delta$-rationalizability by the iterative elimination of actions that are never $\Delta$-best response, which is mathematically equivalent to Definition~\ref{def:ration} (see Appendix~\ref{sec:neverbr}).}.
Since one cannot eliminate all the actions of a player, $|\cup_{i=1}^N \ca_i \setminus E_L|\ge N$, which further implies $L\le N(A-1)<NA$.





\subsection{Equilibria in games}


We consider three common learning objectives, namely Nash Equilibrium (NE), Correlated Equilibrium (CE) and Coarse Correlated Equilibrium (CCE).
\begin{definition}[Nash Equilibrium]
A strategy profile 
 $(x_1,\cdots,x_N)$ is an $\epsilon$-Nash equilibrium if
$$u_i(x_i,x_{-i}) \ge u_i(a, x_{-i})-\epsilon, \forall a \in \ca_i, \forall i\in [N].$$
\end{definition}

\begin{definition}[Correlated Equilibrium]
A correlated strategy $\Pi \in \Delta(\ca)$ is 
an $\epsilon$-correlated equilibrium if $\forall i\in [N], \forall \phi:\ca_i\to \ca_i$,
\begin{equation*}
   \sum_{a_i\in \ca_i, a_{-i}\in \ca_{-i}} \Pi(a_i,a_{-i})u_i(a_i,a_{-i}) \ge
   \sum_{a_i\in \ca_i, a_{-i}\in \ca_{-i}} \Pi(a_i,a_{-i})u_i(\phi(a_i),a_{-i}) - \epsilon.
\end{equation*}
\end{definition}

\begin{definition}[Coarse Correlated Equilibrium]
A correlated strategy $\Pi \in \Delta(\ca)$ is an $\epsilon$-CCE if $\forall i\in [N], \forall a'\in  \ca_i$,
\begin{equation*}
   \sum_{a_i\in \ca_i, a_{-i}\in \ca_{-i}} \Pi(a_i,a_{-i})u_i(a_i,a_{-i}) \ge
   \sum_{a_i\in \ca_i, a_{-i}\in \ca_{-i}} \Pi(a_i,a_{-i})u_i(a',a_{-i}) - \epsilon.
\end{equation*}
\end{definition}
When $\epsilon=0$, the above definitions give exact Nash equilibrium, correlated equilibrium, and coarse correlated equilibrium, respectively. It is well known that $\epsilon$-NE are $\epsilon$-CE, and $\epsilon$-CE are $\epsilon$-CCE.


\subsection{Connection between Equilibria and Rationalizability}
\paragraph{Exact equilibria.} It is known that all actions in the support of an exact NE and CE are rationalizable~\citep[Lemma 56.2]{osborne1994course}. However, one can easily construct an exact CCE that is fully supported on dominated (hence, unrationalizable) actions (see \emph{e.g.}~\citet[Fig. 3]{viossat2013no}). 

\paragraph{Approximate equilibria.} For $\epsilon$-Nash equilibrium with $\epsilon<{\rm poly}(\Delta, 1/N, 1/A)$), one can show that the equilibrium is still mostly supported on rationalizable actions. 
\begin{prop}
\label{prop:nash}
If $x^*=(x^*_1,\cdots,x^*_N)$ is an $\epsilon$-Nash with $\epsilon<\frac{\Delta^2}{24N^2A}$,
$\forall i$, $\Pr_{a\sim x^*_i}[a\in E_{L}] \le \frac{2L\epsilon}{\Delta}.$
\end{prop}
Therefore, for two-player zero-sum games\footnote{For multiplayer general-sum games, finding approximate NE is computationally hard and takes $\Omega(2^N)$ samples in worst case.}, it is possible to run an approximate NE solver and automatically find a rationalizable $\epsilon$-NE. However, this method will induce a rather slow rate, and we will provide a much more efficient algorithm for finding rationalizable $\epsilon$-NE in Section~\ref{sec:cce}.

The connection between CE and rationalizability becomes quite different when it comes to approximate equilibria, which are inevitable in the presence of noise. As shown by~\citet[Theorem 1]{wu2021multi}, an $\epsilon$-CE can be entirely supported on iteratively dominated actions, unless $\epsilon=O(2^{-A})$. In other words, rationalizability is not guaranteed by running an approximate CE solver unless with an extremely high accuracy. Furthermore, since exact CCE can be fully supported on dominated actions, so is $\epsilon$-CCE regardless how small $\epsilon$ is. Therefore, efficiently finding $\epsilon$-CE and $\epsilon$-CCE that are simultaneously rationalizable remains a challenging open problem.

\section{Warm-up: Learning Rationalizable Equilibria Using $A^N$ Samples}

\label{sec:naive}

A direct approach to learn rationalizable equilibria is to achieve this goal by two steps: (1) identify the set of \emph{all} rationalizable actions, and (2) learn equilibria in the subgame restricted to these rationalizable actions. In this section, we provide algorithm and analysis using this approach. We will show such two-step approach inevitably incur $A^{\Omega(N)}$  sample complexity, which is exponential in the number of players $N$.


Formally, we say a set $\mathcal{G}$ is a set of all rationalizable actions \emph{up to tolerance $\Delta$} if and only if (a) $\mathcal{G}$ contains all $0$-rationalizable actions; (b) all actions in $\mathcal{G}$ are $\Delta$-rationalizable.

The algorithm we use to realize the two-step approach is given in Algorithm \ref{alg:naive}, 
which first enumerates all possible action profiles to build an empirical game model, and then performs IDE on this empirical game. When $M$ is sufficiently large, we can guarantee that the set of remaining actions $\tilde{\cA}_1 \times \cdots \times \tilde{\cA}_N$ is a set of all rationalizable actions up to tolerance $\Delta$.
The final step is to learn an $\epsilon$-CCE/CE on the subgame, which will be an $\epsilon$-CCE/CE of the original game as well, which can be done efficiently using FTRL-based algorithms~\citep{cesa2006prediction, blum2007external}. Specifically, we will use Algorithm 5 and 6 by~\citet{jin2021v}, which could learn $\epsilon$-CCE and $\epsilon$-CE using $\tilde{O}\left(\frac{A}{\epsilon^2}\right)$ and $\tilde{O}\left(\frac{A^2}{\epsilon^2}\right)$ samples respectively. The soundness of this algorithm is guaranteed by Proposition~\ref{prop:naive}.

\begin{algorithm}[t]
	\caption{Direct Approach for Learning Rationalizable $\epsilon$-CCE/CE\label{alg:naive}}
	\begin{algorithmic}[1]
		\FORALL{$a=(a_1,a_2,\ldots,a_N)\in\cA$}
		    \STATE Play $a$ for $M$ times
		    \STATE For all $i\in[N]$, compute player $i$'s average payoff $\hat{u}_i(a)$
	  \ENDFOR
      \STATE  Eliminate all $(\Delta/2)$ - IDAs in the empirical game $\{\hat{u}_i\}_{i=1}^N$, and denote the remaining actions as $\tilde{\cA}_i$
      \STATE Find an $\epsilon'$-CCE/CE $\Pi$ of the subgame restricted to $\tilde{\cA}_1 \times \cdots \times \tilde{\cA}_N$ \label{line:oracle}
      \RETURN $\Pi$
	\end{algorithmic}
\end{algorithm}
\begin{prop}
\label{prop:naive}
With parameter $M=\lceil\frac{256\log(1/\delta')}{\Delta^2}\rceil$ where $\delta'=\delta/(A^N N)$, with probability at least $1-\delta$, the output strategy of Algorithm~\ref{alg:naive} is a $\Delta$-rationalizable $\epsilon$-CCE/CE. The total sample complexity is
\[
\tilde{O}\left(\frac{N\cdot A^N}{\Delta^2}+\frac{A}{\epsilon^2}\right) \text{(for CCE)},\ \tilde{O}\left(\frac{N\cdot A^N}{\Delta^2}+\frac{A^2}{\epsilon^2}\right)\text{(for CE)}.
\]
\end{prop}
Unfortunately, the sample complexity of this algorithm is \emph{exponential} in the number of players. As we will see, this exponential dependency is intrinsic to any algorithms which rely on identifying the set of \emph{all} rationalizable actions up to certain tolerance. In fact, we can consider a even simpler task of \emph{deciding} whether a given action $a$ is rationalizable up to tolerance $\Delta$, i.e., output (a) \emph{yes} if $a$ is $0$-rationalizable; (b) \emph{no} if $a$ is not $\Delta$-rationalizable; (c) \emph{arbitrary answer} otherwise.
Our next proposition shows that even this simpler task requires a number of samples exponential in the number of players. The intuition behind this hardness result is that verifying action dominance requires the enumeration of the joint action space of other players, which is exponentially large.
\begin{prop}
\label{prop:decidehard}
For any $\Delta<0.1$, there exists a class of games where deciding whether a given action is rationalizable up to tolerance $\Delta$ with $0.9$ probability needs $\Omega(A^{N-1}\Delta^{-2})$ samples.
\end{prop}
\section{Learning Rationalizable Action Profiles}
\label{sec:oneaction}

To avoid the expoenential dependency in the number of players, we develop efficient approaches for learning rationalizable equilibria \emph{without} identifying the set of all rationalizable actions. In this section, we first consider an easier task---finding one rationalizable action profile that is not necessarily an equilibirum.
Formally, we say a action profile $(a_1, \ldots, a_N)$ is rationalizable if for all $i \in [N]$, $a_i$ is a rationalizable action.
This is arguably one of the most fundamental tasks regarding rationalizability.
For \emph{mixed-strategy dominance solvable} games~\citep{alon2021dominance}, the unique rationalizable action profile will be the unique NE and also the unique CE of the game. Therefore this easier task \emph{per se} is still of practical importance.

This section provides a sample-efficient algorithm which finds a rationalizable action profile using only $\tilde{O}\left(\frac{LNA}{\Delta^2}\right)$ samples. This algorithm will also serve as an important subroutine for algorithms finding rationalizable CCE/CE in the later sections.

The intuition behind this algorithm is simple: if an action profile $a_{-i}$ can survive $l$ rounds of IDE, then its best response $a_i$ (i.e., $\argmax_{a\in{\cA}_i}{u}_i(a, a_{-i})$) can survive at least $l+1$ rounds of IDE, since the action $a_i$ can only be eliminated after some actions in $a_{-i}$ are eliminated. Concretely, we start from an arbitrary action profile $(a_1^{(0)},\ldots,a_N^{(0)})$. In each round $l\in[L]$, we compute the (empirical) best response of $a_{-i}^{(l-1)}$ for each $i\in[N]$, and use those best responses to construct a new action profile $(a_1^{(l)},\ldots,a_N^{(l)})$. By constructing iterative best responses, we will end up with an action profile that can survive $L$ rounds of IDE, which means surviving any number of rounds of IDE according to the definition of $L$.
The full algorithm is presented in Algorithm~\ref{alg:findaction}, for which we have the following theoretical guarantee. 

\begin{algorithm}[t]
	\caption{Iterative Best Response\label{alg:findaction}}
	\begin{algorithmic}[1]
		\STATE \textbf{Initialization:} choose $a_i^{(0)} \in \cA_i$ arbitrarily for all $i\in[N]$
		\FOR{$l=1,\cdots,L$}
		\FOR{$i\in [N]$}
		    \STATE For all $ a\in \cA_i$, play $(a, a^{(l-1)}_{-i})$ for $M$ times, compute player $i$'s average payoff  $\hat{u}_i(a, a^{(l-1)}_{-i})$ 
          \STATE 	Set $a^{(l)}_i\leftarrow\argmax_{a\in{\cA}_i}\hat{u}_i(a, a^{(l-1)}_{-i})$ \COMMENT{ Computing the empirical best response}
		\ENDFOR
		\ENDFOR
		\RETURN{$(a^{(L)}_1,\cdots,a^{(L)}_N)$}
	\end{algorithmic}
\end{algorithm}
\begin{theorem}
\label{thm:ide_upper}
With $M=\left\lceil\frac{16\ln(LNA/\delta)}{\Delta^2}\right\rceil$, with probability $1-\delta$, Algorithm~\ref{alg:findaction} returns an action profile that is $\Delta$-rationalizable using a total of $\tilde{O}\left(\frac{LNA}{\Delta^2}\right)$ samples.
\end{theorem}


\citet{wu2021multi} provide the first polynomial sample complexity results for finding rationalizable action profiles. They prove that their algorithm Exp3-DH is able to find a distribution with $1-\zeta$ fraction supported on $\Delta$-rationalizable actions 
using $\tilde{O}\left(\frac{L^{1.5}N^3A^{1.5}}{\zeta^3\Delta^3}\right)$
samples under bandit feedback\footnote{\citet{wu2021multi}'s result allows trade-off between variables via different choice of algorithmic parameters. However, a $\zeta^{-1}\Delta^{-3}$ factor is unavoidable regardless of choice of parameters. }. 
Compared to their result, our sample complexity bound $\Tilde{O}\left(\frac{LNA}{\Delta^2}\right)$ has more favorable dependence on all problem parameters, and our algorithm will output a distribution that is fully supported on rationalizable actions (thus has no dependence on $\zeta$).  


We further complement Theorem~\ref{thm:ide_upper} with a sample complexity lower bound showing that the linear dependency on $N$ and $A$ are optimal. This lower bound suggests that the $\tilde{O}\left(\frac{LNA}{\Delta^2}\right)$ upper bound is tight up to logarithmic factors when $L=O(1)$, and we conjecture that this is true for general $L$.

\begin{theorem}
\label{thm:ide_lower}
Even for games with $L\le 2$,  any algorithm that returns a $\Delta$-rationalizable action profile with $0.9$ probability needs $\Omega\left(\frac{NA}{\Delta^{2}}\right)$ samples.
\end{theorem}
\begin{conjecture}
\label{conjecture:minimax}
The minimax optimal 
sample complexity for finding a $\Delta$-rationalizable action profile is $\Theta\left(\frac{LNA}{\Delta^2}\right)$ for games with minimum elimination length $L$.
\end{conjecture}

\section{Learning Rationalizable Coarse Correlated Equilibria (CCE)}
\label{sec:cce}
In this section we introduce our algorithm for efficiently learning rationalizable CCEs. The high-level idea is to run no-regret Hedge-style algorithms for every player, while constraining the strategy inside the rationalizable region. Our algorithm is motivated by the fact that the probability of playing a dominated action will decay exponentially over time in the Hedge algorithm for adversarial bandit under full information feedback~\citep{cohen2017hedging}. The full algorithm description is provided in Algorithm \ref{alg:CCE-adaptive}, and here we explain several key components in our algorithm design.

\textbf{Correlated exploration scheme.} In the bandit feedback setting, standard exponential weights algorithms such as EXP3.IX
require \emph{importance sampling} and \emph{biased estimators} to derive a high-probability regret bound~\citep{neu2015explore}. However, such bias could cause a dominating strategy to lose its advantage. In our algorithm we adopt a correlated exploration scheme, which essentially simulates each player's full information feedback by bandit feedback using $NA$ samples.
Specifically, at every time step $t$, the players take turn to enumerate their action set, while the other players fix their strategies according to Hedge. 
It is important to note that such correlated scheme \emph{does not} require any communication between the players---the players can schedule the whole process before the game starts. 


\begin{figure}[t]
\begin{algorithm}[H]
	\caption{Hedge for Rationalizable $\epsilon$-CCE
	\label{alg:CCE-adaptive}}
	\begin{algorithmic}[1]
       \STATE $(a^\*_1,\cdots,a^\*_N)\leftarrow\texttt{Algorithm~\ref{alg:findaction}}$
       \STATE For all $i\in[N]$, initialize $\theta_i^{(1)}(\cdot)\leftarrow \one[\cdot = a^\*_i]$
       \FOR{$t=1,\cdots,T$}
            \FOR{$i=1,\cdots,N$}
           \STATE  For all $a\in\cA_i$, play $(a,\theta_{-i}^{(t)})$ for $M_t$ times, compute player $i$'s average payoff $u_i^{(t)}(a)$  
		    \STATE Set $\theta^{(t+1)}_i(\cdot) \propto \exp\left(\eta_{t}\sum_{\tau=1}^{t} u_i^{(\tau)}(\cdot)\right)$
		    \ENDFOR
		\ENDFOR
	  \STATE For all $t\in[T]$ and $i\in[N]$, eliminate all actions in $\theta_i^{(t)}$ with probability smaller than $p$, then renormalize the vector to simplex as $\bar{\theta}_i^{(t)}$	\STATE \textbf{output:} $\left(\sum_{t=1}^T \otimes_{i=1}^n \bar{\theta}^{(t)}_i\right) / T$
	\end{algorithmic}
\end{algorithm}
\end{figure}

\textbf{Rationalizable initialization and variance reduction.} We use Algorithm~\ref{alg:findaction}, which learns a rationalizable action profile, to give the strategy for the first round. By carefully preserving the disadvantage of any iteratively dominated action, we keep the iterates inside the rationalizable region throughout the whole learning process. To ensure this for every iterate with high probability, a minibatch is used to reduce the variance of the estimator.

\textbf{Clipping.} In the final step, we clip all actions with small probabilities, so that iteratively dominated actions do not appear in the output. The threshold is small enough to not affect the $\epsilon$-CCE guarantee.

\subsection{Theoretical Guarantee}
In Algorithm~\ref{alg:CCE-adaptive}, we choose parameters in the following manner: \begin{equation}
\label{eqn:para_CCE}
\eta_t=\max\left\{\sqrt{\frac{\ln A}{t}},\frac{4\log(1/p)}{\Delta t}\right\}, M_t=\left\lceil\frac{64\log(ANT/\delta)}{\Delta^2 t}\right\rceil, \text{ and }p=\frac{\min\{\epsilon,\Delta\}}{8AN}.
\end{equation}
Note that our learning rate can be bigger than the standard learning rate in FTRL algorithms when $t$ is small. The purpose is to guarantee the rationalizability of the iterates from the beginning of the learning process. As will be shown in the proof, this larger learning rate will not hurt the final rate. We now state the theoretical guarantee for Algorithm~\ref{alg:CCE-adaptive}.
\begin{theorem}
\label{thm:cce_main}
With parameters chosen as in~Eq.\eqref{eqn:para_CCE} , after
$T=\tilde{O}\left(\frac{1}{\epsilon^2}+\frac{1}{\epsilon\Delta}\right)$ rounds, with probability $1-3\delta$, the output strategy
of Algorithm~\ref{alg:CCE-adaptive} is a $\Delta$-rationalizable $\epsilon$-CCE. 
The total sample complexity is \[\tilde{O}\left(\frac{LNA}{\Delta^2} + \frac{NA}{\epsilon^2}\right).\]
\end{theorem}
Here, by  $\Delta$-rationalizable $\epsilon$-CCE, we mean an $\epsilon$-CCE that plays only $\Delta$-rationalizable actions a.s.
\begin{remark} Due to our lower bound (Theorem \ref{thm:ide_lower}), an $\tilde{O}(\frac{NA}{\Delta^2})$ term is unavoidable since learning a rationalizable action profile is an easier task than learning rationalizable CCE.
Based on our Conjecture~\ref{conjecture:minimax}, the additional $L$ dependency is also likely to be inevitable.
On the other hand, learning an $\epsilon$-CCE alone only requires $\tilde{O}(\frac{A}{\epsilon^2})$ samples, where as in our bound we have a larger $\tilde{O}(\frac{NA}{\epsilon^2})$ term. 
The extra $N$ factor is a consequence of our correlated exploration scheme in which only one player explores at a time. Removing this $N$ factor might require more sophisticated exploration methods and utility estimators, which we leave as future work.
\end{remark}

\subsubsection{Overview of the analysis}
We give an overview of our analysis of Algorithm~\ref{alg:CCE-adaptive} below. The full proof is deferred to Appendix~\ref{sec:cceproof}.

\textbf{Step 1: Ensure rationalizability.} We will first show that rationalizability is preserved at each iterate, \emph{i.e.}, actions in $E_L$ will be played with low probability across all iterates. Formally,
\begin{lem}
\label{lem:CCE_lowprob}
With probability at least $1-2\delta$, for all $t\in[T]$ and all $i\in[N]$, $a_i\in\cA_i\cap E_L$,
\begin{align*}
    \textstyle{ \theta_i^{(t)}(a_i)\leq p.}
\end{align*}
\end{lem}
where $p$ is defined in \eqref{eqn:para_CCE}. Lemma~\ref{lem:CCE_lowprob} guarantees that, after the clipping in Line 7 of Algorithm~\ref{alg:CCE-adaptive}, the output correlated strategy be $\Delta$-rationalizable.

We proceed to explain the main idea for proving Lemma~\ref{lem:CCE_lowprob}.
A key observation is that the set of rationalizable actions, $\cup_{i=1}^n \cA_i\setminus E_L$, is closed under best response---for the $i$-th player, as long as the other players continue to play actions in $\cup_{j\neq i} \cA_j\setminus E_L$, actions in $\cA_i\cap E_L$ will suffer from excess losses each round in an exponential weights style algorithm.
Concretely, for any $a_{-i}\in(\prod_{j\neq i}\cA_j)\setminus E_L$ and any iteratively dominated action $a_i\in \cA_i\cap E_L$, there always exists  $x_i\in\Delta(\cA_i)$ such that
\[
u_i(x_i,a_{-i})\geq u_i(a_i,a_{-i}) + \Delta.
\]
With our choice of $p$ in Eq.~(\ref{eqn:para_CCE}), if other players choose their actions from $\cup_{j\neq i}\cA_j \setminus E_L$ with probability $1-pAN$, we can still guarantee an excess loss of $\Omega(\Delta)$. It follows that
\[
\sum_{\tau=1}^{t} u_i^{(\tau)}(x_i)-\sum_{\tau=1}^{t} u_i^{(\tau)}(a_i)\geq \Omega(t\Delta) - \text{Sampling Noise}.
\]
However, this excess loss can be obscured by the noise from bandit feedback when $t$ is small. Note that it is crucial that the statement of Lemma~\ref{lem:CCE_lowprob} holds for all $t$ due to the inductive nature of the proof.
As a solution, we use a minibatch of size $M_t=\tilde{O}\left(\lceil\frac{1}{\Delta^2 t}\right\rceil)$ in the $t$-th round to reduce the variance of the payoff estimator $u_i^{(t)}$. The noise term can now be upper-bounded with Azuma-Hoeffding by
\[
\text{Sampling Noise}\leq \tilde{O}\left(\sqrt{\sum_{\tau=1}^t \frac{1}{M_t}}\right)\leq O(t\Delta),
\]
Combining this with our choice of the learning rate $\eta_t$ gives
\begin{equation}
\label{eqn:1}
\eta_t\left(\sum_{\tau=1}^{t} u_i^{(\tau)}(x_i)-\sum_{\tau=1}^{t} u_i^{(\tau)}(a_i)\right)\gg 1.
\end{equation}
{By the update rule of the Hedge algorithm}, this implies that $\theta_i^{(t+1)}(a_i)\leq p$, which enables us to complete the proof of Lemma~\ref{lem:CCE_lowprob} via induction on $t$.

\textbf{Step 2: Combine with no-regret guarantees.} Next, we prove that the output strategy is an $\epsilon$-CCE. For a player $i\in[N]$, the regret is defined as
$
{\rm Regret}_T^i=\max_{\theta\in\Delta({\cA_i})}\sum_{t=1}^T \langle  u^{(t)}_i, \theta-\theta_i^{(t)}\rangle.
$
We can obtain the following regret bound by standard analysis of FTRL with changing learning rates.
\begin{lemma}
\label{lem:CCE-Regret}
For all $i\in[N]$, ${\rm Regret}_T^i\leq\tilde{O}\left(\sqrt{T}+\frac{1}{\Delta}\right)$.
\end{lemma}
Here the additive $1/\Delta$ term is the result of our larger $\tilde{O}(\Delta^{-1}t^{-1})$ learning rate for small $t$. It follows from Lemma~\ref{lem:CCE-Regret} that $T=\tilde{O}\left(\frac{1}{\epsilon^2}+\frac{1}{\Delta\epsilon}\right)$ suffices to guarantee that the correlated strategy $\frac{1}{T}\left(\sum_{t=1}^T \otimes_{i=1}^n {\theta}^{(t)}_i\right)$ is an $(\epsilon/2)$-CCE. Since $pNA=O(\epsilon)$, the clipping step only minorly affects the CCE guarantee and the clipped strategy  $\frac{1}{T}\left(\sum_{t=1}^T \otimes_{i=1}^n \bar{\theta}^{(t)}_i\right)$ is an $\epsilon$-CCE.

\subsection{Learning rationalizable Nash equilibrium in zero-sum games}
Algorithm~\ref{alg:CCE-adaptive} can also be applied to two-player zero-sum games to learn a rationalizable $\epsilon$-NE efficiently.
Note that in two-player zero-sum games, the marginal distribution of an $\epsilon$-CCE is guaranteed to be a $2\epsilon$-Nash (see, \emph{e.g.}, Proposition 9 in \citet{bai2020near}). Hence direct application of Algorithm~\ref{alg:CCE-adaptive} to a zero-sum game gives the following sample complexity bound.
\begin{coro}
In a two-player zero-sum game, the sample complexity for finding a $\Delta$-rationalizable $\epsilon$-Nash with Algorithm~\ref{alg:CCE-adaptive} is $\tilde{O}\left(\frac{LA}{\Delta^2} + \frac{A}{\epsilon^2}\right)$.
\end{coro}
This result improves over a direct application of Proposition~\ref{prop:nash}, which gives $\Tilde{O}\left(\frac{A^3}{\Delta^4}+\frac{A}{\epsilon^2}\right)$ sample complexity and produces an $\epsilon$-Nash that could still take unrationalizable actions with positive probability.

\section{Learning Rationalizable Correlated Equilibrium (CE)}
\label{sec:ce}
In order to extend our results on $\epsilon$-CCE to $\epsilon$-CE, a natural approach would be augmenting Algorithm~\ref{alg:CCE-adaptive} with the celebrated  Blum-Mansour reduction~\citep{blum2007external} from swap-regret to external regret. In this reduction, one maintains $A$ instances of a no-regret algorithm $\{\rm Alg_1,\cdots,Alg_A\}$. In iteration $t$, the player would stack the recommendations of the $A$ algorithms as a matrix, denoted by $\hat \theta^{(t)}\in\R^{A\times A}$, and compute its eigenvector $\theta^{(t)}$ as the randomized strategy in round $t$. After observing the actual payoff vector $u^{(t)}$, it will pass the weighted payoff vector $\theta^{(t)}(a)u^{(t)}$ to algorithm {$\rm Alg_a$} for each $a$. In this section, we focus on a fixed player $i$, and omit the subscript $i$ when it's clear from the context.

Applying this reduction to Algorithm~\ref{alg:CCE-adaptive} directly, however, would fail to preserve rationalizability since the weighted loss vector $\theta^{(t)}(a)u^{(t)}$ admit a smaller utility gap $\theta^{(t)}(a)\Delta$. Specifically, consider an action $b$ dominated by a mixed strategy $x$. In the payoff estimate of instance $a$, 
\begin{align}
    \sum_{\tau=1}^t \theta^{(\tau)}(a)\left(u^{(\tau)}(b)-u^{(\tau)}(x)\right) \gtrsim \Delta\sum_{\tau=1}^t \theta^{(\tau)}(a) - \sqrt{\sum_{\tau=1}^t \frac{1}{M^{(\tau)}}} \ngeq 0,\label{eq:ce-azuma1}
\end{align}
which means that we cannot guarantee the elimination of IDAs every round as in Eq~\eqref{eqn:1}.

In Algorithm~\ref{alg:CE-adaptive}, we address this by making $\sum_{\tau=1}^t \theta^{(\tau)}(a)$ play the role as $t$, tracking the progress of each no-regret instance separately. In time step $t$, we will compute the average payoff vector $u^{(t)}$ based on $M^{(t)}$ samples; then as in the Blum-Mansour reduction, we will update the $A$ instances of Hedge with weighted payoffs $\theta^{(t)}(a)u^{(t)}$ and will use the eigenvector of $\hat\theta$ as the strategy for the next round. The key detail here is our choice of parameters, which adapts to the past strategies $\{\theta^{(\tau)}\}_{\tau=1}^t$:
\begin{equation}
\label{eq:ce-param}
    M^{(t)}_i:=\left\lceil\max_a\frac{64\theta_i^{(t)}(a)}{\Delta^2\cdot  \sum_{\tau=1}^t \theta^{(\tau)}_i(a)}\right\rceil,\eta_{t,i}^a:=\max\left\{\frac{2\ln(1/p)}{\Delta\sum_{\tau=1}^{t}\theta_i^{(\tau)}(a)},\sqrt{\frac{A \ln A}{t}}\right\},p=\frac{\min\{\epsilon,\Delta\}}{8AN}.
\end{equation}
Compared to Eq~\eqref{eqn:para_CCE}, we are essentially replacing $t$ with an adaptive $\sum_{\tau=1}^t \theta^{(\tau)}(a)$. We can now improve (\ref{eq:ce-azuma1}) to
\begin{equation}
\label{eq:ceazuma2}
    \sum_{\tau=1}^t \theta^{(\tau)}(a)\left(u^{(\tau)}(b)-u^{(\tau)}(x)\right) \gtrsim \Delta\sum_{\tau=1}^t \theta^{(\tau)}(a) - \sqrt{\sum_{\tau=1}^t \frac{\theta^{(\tau)}(a)^2}{M^{(\tau)}}} \gtrsim \Delta\sum_{\tau=1}^t \theta^{(\tau)}(a) .
\end{equation}

This together with our choice of $\eta_t^a$ allows us to ensure the rationalizability of every iterate. The full algorithm is presented in Algorithm~\ref{alg:CE-adaptive}.

\begin{algorithm}[t]
	\caption{Adaptive Hedge for Rationalizable $\epsilon$-CE \label{alg:CE-adaptive}}
	\begin{algorithmic}[1]
		\STATE $(a^\*_1,\cdots,a^\*_N)\leftarrow\texttt{Algorithm~\ref{alg:findaction}}$
		\STATE For all $i\in[N]$, initialize $\theta_i^{(1)}\gets (1-|\ca_i|p) \one[\cdot = a^\*_i]+ p\one$
 	   \FOR{$t=1,2,\ldots,T$} \FOR{$i=1,2,\ldots,N$}
 	   \STATE  For all $a\in\cA_i$, play $(a,\theta_{-i}^{(t)})$ for $M^{(t)}_i$ times, compute player $i$'s average payoff $u_i^{(t)}(a)$ 
	    \STATE For all $b\in\cA_i$, set $\hat\theta_i^{(t+1)}(\cdot|b) \propto \exp\left(\eta_{t,i}^b  \sum_{\tau=1}^t u_i^{(\tau)}(\cdot)\theta_i^{(\tau)}(b)\right)$
	    \STATE Find $\theta_i^{(t+1)}\in\Delta({\cA_i})$ such that $\theta_i^{(t+1)}(a) = \sum_{b\in \ca_i}\hat\theta^{(t+1)}_i(a|b)\theta^{(t+1)}_i(b)$
		\ENDFOR
		\ENDFOR
		\STATE For all $t\in[T]$ and $i\in[N]$, eliminate all actions in $\theta_i^{(t)}$ with probability smaller than $p$, then renormalize the vector to simplex as $\bar{\theta}_i^{(t)}$	\STATE \textbf{output:} $\left(\sum_{t=1}^T \otimes_{i=1}^n \bar{\theta}^{(t)}_i\right) / T$
	\end{algorithmic}
\end{algorithm}

We proceed to our theoretical guarantee for Algorithm~\ref{alg:CE-adaptive}. The analysis framework is largely similar to that of Algorithm~\ref{alg:CCE-adaptive}. Our choice of $M_i^{(t)}$ is sufficient to ensure $\Delta$-rationalizability via Azuma-Hoeffding inequality, while swap-regret analysis of the algorithm proves that the average (clipped) strategy is indeed an $\epsilon$-CE. The full proof is deferred to Appendix~\ref{sec:ce-proof}. 
\begin{theorem}
\label{thm:ce_main}
With parameters in Eq.~\eqref{eq:ce-param}, after $T=\tilde{O}\left(\frac{A}{\epsilon^2}+\frac{A}{\Delta^2}\right)$ rounds, with probability $1-3\delta$, the output strategy of Algorithm~\ref{alg:CE-adaptive} is a $\Delta$-rationalizable $\epsilon$-CE . The total sample complexity is
\begin{align*}
    \Tilde{O}\left(\frac{LNA}{\Delta^2}+\frac{NA^2}{\min\{\Delta^2,\epsilon^2\}}\right).
\end{align*}
\end{theorem}
Compared to Theorem~\ref{thm:cce_main}, our second term has an additional $A$ factor, which is quite reasonable considering that algorithms for learning $\epsilon$-CE take $\tilde{O}(A^2\epsilon^{-2})$ samples, also $A$-times larger than the $\epsilon$-CCE rate.

\section{Reduction-based Algorithms} 
\label{sec:reduction}
While Algorithm~\ref{alg:CCE-adaptive} and~\ref{alg:CE-adaptive} make use of one specific no-regret algorithm, namely Hedge (Exponential Weights), in this section, we show that arbitrary algorithms for finding CCE/CE can be augmented to find rationalizable CCE/CE. The sample complexity obtained via this reduction is comparable with those of Algorithm~\ref{alg:CCE-adaptive} and~\ref{alg:CE-adaptive} when $L=\Theta(NA)$, but slightly worse when $L\ll NA$. 
Moreover, this black-box approach would enable us to derive algorithms for rationalizable equilibria with more desirable qualities, such as last-iterate convergence, when using equilibria-finding algorithms with these properties.

 \begin{algorithm}[t]
	\caption{Rationalizable $\epsilon$-CCE via Black-box Reduction \label{alg:cce-reduction}}
	\begin{algorithmic}[1]
		\STATE $(a^\*_1,\cdots,a^\*_N)\leftarrow\texttt{Algorithm~\ref{alg:findaction}}$
		\STATE For all $i\in[N]$, initialize $\ca_i^{(1)}\gets \{a^\*_i\}$ 
		\FOR{$t=1,2,\ldots$}
		\STATE Find an $\epsilon'$-CCE $\Pi$ with black-box algorithm $\mathcal{O}$ in the subgame $\Pi_{i\in[N]}{\cA}_i^{(t)}$
		\STATE $\forall i\in [N], a'_i\in\cA_i$, evaluate $ u_i(a'_i, \Pi_{-i})$ for $M$ times and compute average $\hat u_i(a'_i, \Pi_{-i})$  
		\FOR{$i\in [N]$}
		    
        	\STATE Let $a_i'\leftarrow\argmax_{a\in{\cA}_i}\hat u_i(a, \Pi_{-i})$ \COMMENT{{ Computing the empirical best response}}
        	\STATE  ${\cA}^{(t+1)}_i\leftarrow {\cA}^{(t)}_i\cup\{a_i'\}$

		\ENDFOR
	    
	    \IF{$\cA_i^{(t)}=\cA_i^{(t+1)}$ for all $i\in[N]$}
		\RETURN $\Pi$
	    \ENDIF
	    \ENDFOR
	\end{algorithmic}
\end{algorithm}
Suppose that we are given a black-box algorithm~$\mathcal{O}$ that finds $\epsilon$-CCE in arbitrary games. We can then use this algorithm in the following ``support expansion'' manner. We start with a subgame of only rationalizable actions, which can be identified efficiently with Algorithm~\ref{alg:findaction}, and call~$\mathcal{O}$ to find an $\epsilon$-CCE $\Pi$ for the subgame. Next, we check for each $i\in[N]$ if the best response to $\Pi_{-i}$ is contained in $\ca_i$. If not, this means that the subgame's $\epsilon$-CCE may not be an $\epsilon$-CCE for the full game; in this case, the best response to $\Pi_{-i}$ would be a rationalizable action that we can safely include into the action set. On the other hand, if the best response falls in $\ca_i$ for all $i$, we can conclude that $\Pi$ is also an $\epsilon$-CCE for the original game. The details are given by Algorithm~\ref{alg:cce-reduction}, and our main theoretical guarantee is the following.

\begin{theorem}
\label{thm:cce-reduction}
With $M = \left\lceil\frac{4\ln\left(2NA/\delta\right)}{\epsilon'^2}\right\rceil$, Algorithm~\ref{alg:cce-reduction} outputs a $\Delta$-rationalizable $\epsilon$-CCE with probability at least $1-2\delta$, using at most $NA$ calls to the black-box CCE algorithm and $\tilde{O}\left(\frac{N^2A^2}{\min\{\epsilon^2,\Delta^2\}}\right)$ additional samples.
\end{theorem}

Similarly, we can develop a reduction scheme for rationalizable $\epsilon$-CE. The algorithm for CE is quite similar to the one for CCE, except now when testing whether a subgame $\epsilon$-CE is an actual $\epsilon$-CE, we need to use the conditional distribution $\Pi|a_i$, which is the conditional distribution of the other players' actions given that player $i$ is told to play $a_i$. The detailed description is given in Algorithm~\ref{alg:ce_2}, and its main theoretical guarantee is the following.

\begin{theorem}
\label{thm:ce-reduction}
With $M = \left\lceil\frac{4\ln\left(2NA^2/\delta\right)}{\epsilon'^2}\right\rceil$, Algorithm~\ref{alg:ce_2} outputs a $\Delta$-rationalizable $\epsilon$-CE with probability at least $1-2\delta$, using at most $NA$ calls to a black-box CE algorithm and $\tilde{O}\left(\frac{N^2A^3}{\min\{\epsilon^2,\Delta^2\}}\right)$ additional samples.
\end{theorem}
\begin{algorithm}[h]
	\caption{Rationalizable $\epsilon$-CE via Black-box Reduction \label{alg:ce_2}}
	\begin{algorithmic}[1]
		 \STATE $(a^\*_1,\cdots,a^\*_N)\leftarrow\texttt{Algorithm~\ref{alg:findaction}}$
		 \STATE For all $i\in[N]$, initialize $\ca_i^{(1)}\gets \{a^\*_i\}$ for all $i\in[N]$
		\FOR{$t=1,2,\ldots$}
		\STATE Find an $\epsilon'$-CE, $\Pi$, in the subgame supported on $\Pi_{i\in[N]}{\cA}_i^{(t)}$
		\STATE $\forall i\in [N], a_i,a'_i\in\cA_i$, sample $ u_i(a'_i, \Pi_{-i}|a_i)$ for $M$ times and compute average $\hat u_i(a'_i, \Pi_{-i}|a_i)$ 
		\FOR{$i\in [N]$}
		    \FOR{$a_i\in {\cA}_i^{(t)}$}
        	\STATE Let	\[
        		a_i'\leftarrow\argmax_{a\in{\cA}_i}\hat u_i(a, \Pi |a_i)\text{\color{blue}// Computing the empirical best response} \]
        	\STATE  ${\cA}^{(t+1)}_i\leftarrow {\cA}^{(t)}_i\cup\{a_i'\}$
		    
		    \ENDFOR
		\ENDFOR
	    
	    \IF{$\cA_i^{(t)}=\cA_i^{(t+1)}$ for all $i\in[N]$}
		\RETURN $\Pi$
	    \ENDIF
	    \ENDFOR
	\end{algorithmic}
\end{algorithm}

\section{Conclusion}

In this paper, we consider two tasks: (1) learning rationalizable action profiles; (2) learning rationalizable equilibria. For task 1, we propose a conceptually simple algorithm whose sample complexity is significantly better than prior work \citep{wu2021multi}. For task 2, we develop the first provably efficient algorithms for learning $\epsilon$-CE and $\epsilon$-CCE that are also {rationalizable}. 
Our algorithms are computationally efficient, enjoy sample complexity that scales polynomially with the number of players and are able to avoid iteratively dominated actions completely. 
Our results rely on several new techniques which might be of independent interests to the community. There remains a gap between our sample complexity upper bounds and the available lower bounds for both tasks, closing which is an important future research problem.

\section*{Acknowledgements}
This work is supported by Office of Naval Research N00014-22-1-2253.
Dingwen Kong is partially supported by the elite undergraduate training program of School of Mathematical Sciences in Peking University.
\bibliography{ref}
\bibliographystyle{iclr2023_conference}

\newpage
\appendix
\section{Further Details on Rationalizability}
\label{sec:detail}

\subsection{Equivalence of Never-best-response and strict dominance}
\label{sec:neverbr}

It is known that for finite normal form games, rationalizable actions are given by iterated elimination of never-best-response actions, which is in fact equivalent to the iterative elimination of strictly dominated actions~\citep[Lemma 60.1]{osborne1994course}. Here, for completeness, we include a proof that the iterative elimination of of actions that are never $\Delta$-best-response gives the same definition as Definition~\ref{def:ration}. Notice that it suffices to show that for every subgame, the set of never $\Delta$-best response actions and the set of $\Delta$-dominated actions are the same.
\begin{proposition}
\label{prop:neverbr}
Suppose that an action $a\in \ca_i$ is never a $\Delta$-best response, \emph{i.e.} $\forall \Pi_{-i} \in \Delta(\prod_{j\neq i}\ca_i)$, $\exists u\in \Delta(\ca_i)$ such that 
\begin{align*}
    u_i\left(a,\Pi_{-i}\right) \le u_i\left(u, \Pi_{-1}\right) - \Delta.
\end{align*}
Then $a$ is also $\Delta$-dominated, \emph{i.e.} $\exists u\in \Delta(\ca_i)$, $\forall \Pi_{-i}\in \Delta(\prod_{j\neq i}\ca_i)$
\begin{align*}
    u_i\left(a,\Pi_{-i}\right) \le u_i\left(u, \Pi_{-1}\right) - \Delta.
\end{align*}
\end{proposition}
\begin{proof}
That $a$ is never a $\Delta$-best response is equivalent to
\begin{align*}
    \min_{\Pi_{-1}}\max_{u}  \left\{u_i\left(a,\Pi_{-i}\right)-u_i\left(u,\Pi_{-1}\right)\right\} \le -\Delta.
\end{align*}
That $a$ is $\Delta$-dominated is equivalent to
\begin{align*}
    \max_{u}\min_{\Pi_{-1}}  \left\{u_i\left(a,\Pi_{-i}\right)-u_i\left(u,\Pi_{-1}\right)\right\} \le -\Delta.
\end{align*}
Equivalence immediately follows from von Neumman's minimax theorem.
\end{proof}

\subsection{Proof of Proposition~\ref{prop:nash}}
\label{sec:nashproof}

\begin{proof}

We prove this inductively with the following hypothesis:
\begin{align*}
\forall l\ge 1, \forall i\in [N], \qquad \sum_{a\in\ca_{i}} x^*_i(a)\cdot \one[a\in E_l] \le \frac{2l\epsilon}{\Delta}.
\end{align*}

\paragraph{Base case:}

By the definition of $\epsilon$-NE, $\forall i\in [N]$, $\forall x'\in \Delta({\ca_i})$,
\begin{align*}
u_i(x^*_i,x^*_{-i}) \ge u_i(x',x^*_{-i}) - \epsilon.
\end{align*}
Note that if $\ta\in E_1 \cap \ca_i$, $\exists x\in \Delta(\ca_i)$ such that $\forall a_{-i}$,
\begin{align*}
u_i(\ta,a_{-i}) \le u_i(x,a_{-i}) - \Delta.
\end{align*}
Therefore if we choose
$$x':= x^*_i - \sum_{a\in \ca_i} \one[a\in E_1]x^*_i(a) \mathbf{e}_a + \sum_{a\in \ca_i} \one[a\in E_1] x^*_i(a)\cdot x(a),$$
that is if we play the dominating strategy instead of the dominated action in $x^*_i$, then
\begin{align*}
 u_i(x',x^*_{-i}) \ge u_i(x^*_i, x^*_{-i}) +  \sum_{a\in \ca_i} x^*_i(a)\cdot \one[a\in E_1] \Delta.
\end{align*}
It follows that 
$$ \sum_{a\in\ca_i} x^*_i(a)\cdot \one[a\in E_1] \le \frac{\epsilon}{\Delta}.$$
\paragraph{Induction step:}
By the induction hypothesis, $\forall i\in [N]$,
$$ \sum_{a\in\ca_i} x^*_i(a)\cdot \one[a\in E_l] \le \frac{2l\epsilon}{\Delta}.$$
Now consider 
\begin{equation*}
    \tx_i := \frac{x^*_i - \sum_{a\in\ca_i} \one[a\in E_l]\cdot x^*_i(a) \mathbf{e}_a}{1- \sum_{a\in\ca_i} \one[a\in E_l]\cdot x^*_i(a) },\tag{$\forall i\in[N]$}
\end{equation*}
which is supported on actions on in $E_l$. The induction hypothesis implies $\Vert \tx_i - x^*_i\Vert_1\le 6l\epsilon/\Delta$. Therefore $\forall i\in [N]$, $\forall a\in\ca_i$,
\begin{align*}
\left| u_i(a, \tx_{-i}) - u_i(a, x^*_{-i}) \right| \le \frac{6Nl \epsilon}{\Delta}.
\end{align*}
Now if $\ta\in \left(E_{l+1}\setminus E_l\right)\cap \ca_i$, since $\tx_{-i}$ is not supported on $E_l$, $\exists x\in\Delta(\ca_i)$ such that
\begin{align*}
	u_i(\ta,\tx_{-i}) \le u_i(x,\tx_{-i}) - \Delta.
\end{align*}
It follows that
\begin{align*}
	u_i(\ta,x^*_{-i}) \le u_i(x,x^*_{-i}) - \Delta +  \frac{12Nl \epsilon}{\Delta} \le  u_i(x,x^*_{-i}) - \frac{\Delta}{2}.
\end{align*}
Using the same arguments as in the base case,
\begin{align*}
\sum_{a\in\ca_i} x^*_i(a)\cdot \one[a\in E_{l+1}\setminus E_l] \le \frac{\epsilon}{\Delta - \frac{12Nl \epsilon}{\Delta}} \le \frac{2\epsilon}{\Delta}.
\end{align*}
It follows that $\forall i\in[N]$,
$$ \sum_{a\in \ca_i} x^*_i(a)\cdot \one[a\in E_{l+1}] \le \frac{2(l+1)\epsilon}{\Delta}.$$
The statement is thus proved via induction on $l$.
\end{proof}

\section{Omitted Proofs in Section~\ref{sec:naive} and~\ref{sec:oneaction}}
\subsection{Proof of Proposition~\ref{prop:naive}}
\begin{lem}
\label{naivelemma}
With probability at least $1-\delta$, for all $i\in[N]$, $\tilde{\cA}_i$ does not contain any $\Delta$-IDA. 
\end{lem}
\begin{proof}
We first present the concentration bound. For $i\in[N]$, and $a\in\cA$, by Hoeffding's inequality we have that with probability at least $1-{\delta'}$,
\[
\left|u_i(a)-\hat{u}_i(a)\right|\leq \sqrt{\frac{4\ln(1/\delta')}{M}}\leq \frac{\Delta}{8}.
\]
Therefore by a union bound we have that with probability at least $1-\delta$, for all  $i\in[N]$, and $a\in\cA$,
\[
\left|u_i(a)-\hat{u}_i(a)\right|\leq \frac{\Delta}{8}.
\]
Therefore we have that any $\Delta$-IDA in the original game is also a $\frac{\Delta}{2}$-IDA in the empirical game, thus will be eliminated from $\tilde{\cA}_i$
\end{proof}
\begin{lem}
\label{naivelemma2}
With probability at least $1-\delta$, $\Pi$ is a $\epsilon$-CCE/CE in the original game.
\end{lem}
\begin{proof}
We condition on the event defined in Lemma~\ref{naivelemma}. For any action $a_i\in\cA_i\setminus\tilde{\cA}_i$, since it is iteratively eliminated in the empirical game, it must be dominated by a mixed strategy $x_i$ over $\tilde{\cA}_i$, for any opponent action in $\Pi_{j\neq i}\tilde{\cA}_j$. Therefore one can check that any $\epsilon$-CCE/CE of the subgame is a $\epsilon$-CCE/CE by definition.
\end{proof}
\begin{proof}[Proof of Proposition~\ref{prop:naive}]
Combining Lemma~\ref{naivelemma} and Lemma~\ref{naivelemma2} completes the proof.
\end{proof}
\subsection{{Proof of Proposition~\ref{prop:decidehard}}}
\begin{proof}
Consider the following $N$-player game denoted by $G_0$ with action set $[A]$:
\begin{align*}
    u_i\left(\cdot\right) &= 0 \tag{$1\le i\le N-1$} \\
    u_{N}\left(a_N\right) &= \Delta\cdot \one[a_N >1].
\end{align*}
Specifically, a payoff with mean $u$ is realized by a skewed Rademacher random variable with $\frac{1+u}{2}$ probability on $+1$ and $\frac{1-u}{2}$ on $-1$. In game $G_0$, clearly for player $N$, the action $1$ is $\Delta$-dominated. However, consider the following game, denoted by $G_{a^*}$  (where $a^*\in [A]^{N-1}$)
\begin{align*}
    u_i\left(\cdot  \right) &= 0, \tag{$1\le i\le N-1$} \\
    u_{N}\left(a_N\right) &= \Delta, \tag{$a_N>1$}\\
    u_N\left(1 , a_{-N}\right) &= 2\Delta\cdot \one[a_{-N} = a^*].
\end{align*}
It can be seen that in game $G_{a^*}$, for player $N$, the action $1$ is \emph{not} dominated or iteratively strictly dominated. 
Therefore, suppose that an algorithm $\mathcal{O}$ is able to decide whether the action $1$ is rationalizable up to tolerance $\Delta$ with $0.9$ accuracy, then its output needs to be False with at least $0.9$  probability in game $G_0$, but True with at least $0.9$ probability in game $G_{a^*}$. By Pinsker's inequality,
\begin{align*}
    {\rm KL}(\mathcal{O}(G_0)||\mathcal{O}(G_{a^*})) \ge 2\cdot 0.8^2 >1,
\end{align*}
where we used $\mathcal{O}(G)$ to denote the trajectory generated by running algorithm $\mathcal{O}$ on game $G$. Meanwhile, notice that $G_0$ and $G_{a^*}$ is different only when the first $N-1$ players play $a^*$. Denote the number of times where the first $N-1$ players play $a^*$ by $n(a^*)$. Using the chain rule of KL-divergence,
\begin{align*}
     {\rm KL}(\mathcal{O}(G_0)||\mathcal{O}(G_{a^*})) &\le \E_{G_0}\left[n(a^*)\right]\cdot {\rm KL\left(\left.{\rm Ber}\left(\frac{1}{2}\right) \right\Vert {\rm Ber}\left(\frac{1+2\Delta}{2}\right)\right)}\\
     &\stackrel{(a)}{\le}  \E_{G_0}\left[n(a^*)\right]\cdot \frac{1}{\frac{1-2\Delta}{2}}\cdot (2\Delta)^2\\
     &\stackrel{(b)}{\le} 10\Delta^2\E_{G_0}\left[n(a^*)\right].
\end{align*}
Here $(a)$ follows from reverse Pinsker's inequality (see \emph{e.g.}~\citet{binette2019note}),  while $(b)$ uses the fact that $\Delta < 0.1$. This means that for any $a^*\in [A]^{N-1}$,
\begin{align*}
    \E_{G_0}\left[n(a^*)\right] \ge \frac{1}{10\Delta^2}.
\end{align*}
It follows that the expected number of samples when running $\mathcal{O}$ on $G_0$ is at least 
\begin{align*}
    \E_{G_0}\left[\sum_{a^*\in [A]^{N-1}}n(a^*)\right] \ge \frac{A^{N-1}}{10\Delta^2}.
\end{align*}
\end{proof}

\subsection{Proof of Theorem~\ref{thm:ide_upper}}
\begin{proof}
We first present the concentration bound. For $l\in[L]$, $i\in[N]$, and $a\in\cA_i$, by Hoeffding's inequality we have that with probability at least $1-\frac{\delta}{LNA}$,
\[
\left|u_i(a,a_{-i}^{(l-1)})-\hat{u}_i(a,a_{-i}^{(l-1)})\right|\leq \sqrt{\frac{4\ln(ANL/\delta)}{M}}\leq \frac{\Delta}{4}.
\]
Therefore by a union bound we have that with probability at least $1-\delta$, for all $l\in[L]$, $i\in[N]$, and $a\in\cA_i$,
\[
\left|u_i(a,a_{-i}^{(l-1)})-\hat{u}_i(a,a_{-i}^{(l-1)})\right|\leq \frac{\Delta}{4}.
\]
We condition on this event for the rest of the proof.

We use induction on $l$ to prove that for all $l\in[L]\cup\{0\}$, $(a^{(l)}_1,\cdots,a^{(l)}_N)$ can survive at least $l$ rounds of IDE.
The base case for $l=0$ directly holds. Now we assume that the case for $1,2,\ldots,l-1$ holds and consider the case of $l$.

For any $i\in[N]$, we show that $a_i^{(l)}$ can survive at least $l$ rounds of IDE. Recall that $a_i^{(l)}$ is the empirical best response, \emph{i.e.}
\[
a_i^{(l)}=\argmax_{a\in{\cA}_i}\hat{u}_i(a, a^{(l-1)}_{-i}).
\]
For any mixed strategy $x_i\in\Delta(\cA_i)$, we have that
\begin{align*}
&u_i(a_i^{(l)},a_{-i}^{(l-1)})-u_i(x_i,a_{-i}^{(l-1)})\\
\geq &\hat{u}_i(a_i^{(l)},a_{-i}^{(l-1)})-\hat{u}_i(x_i,a_{-i}^{(l-1)})-\left|u_i(a^{(l)}_i,a_{-i}^{(l-1)})-\hat{u}_i(a^{(l)}_i,a_{-i}^{(l-1)})\right|-\left|u_i(x_i,a_{-i}^{(l-1)})-\hat{u}_i(x_i,a_{-i}^{(l-1)})\right|\\
\geq & 0-\frac{\Delta}{4}-\frac{\Delta}{4}=-\frac{\Delta}{2}.
\end{align*}
Since actions in $a_{-i}^{(l-1)}$ can survive at least $l-1$ rounds of $\Delta$-IDE, $a_i^{(l)}$ cannot be $\Delta$-dominated by $x_i$ in rounds $1,\cdots,l$. Since $x_i$ can be arbitrarily chosen, $a_i^{(l)}$ can survive at least $l$ rounds of $\Delta$-IDE.
We can now ensure that the output $(a^{(L)}_1,\cdots,a^{(L)}_N)$ survives $L$ rounds of $\Delta$-IDE, which is equivalent to $\Delta$-rationalizability (see Definition~\ref{def:ration}).

The total number of samples used is
\[
LNA\cdot M=\tilde{O}\left(\frac{LNA}{\Delta^2}\right).
\]
\end{proof}

\subsection{Proof of Theorem~\ref{thm:ide_lower}}
\begin{proof}
Without loss of generality, assume that $\Delta<0.1$. Consider the following instance where $\cA_1=\cdots=\cA_N=[A]$:
\begin{align*}
    u_i(a_i) &= \Delta\cdot \one[a_i=1], \tag{$i\neq j$}\\
    u_{j}(a_j,a_{-j}) &= \begin{cases}
    \Delta\cdot \one[a_j=1] & (a_{-j}\neq \{1\}^{N-1})\\
   \Delta\cdot \one[a_j=1] + 2\Delta\cdot \one[a_j=a] & (a_{-j} = \{1\}^{N-1})
    \end{cases}.
\end{align*}
Denote this instance by $G_{j,a}$. Additionally, define the following instance $G_0$:
\begin{align*}
    u_i(a_i) &= \Delta\cdot \one[a_i=1]. \tag{$\forall i\in [N]$}
\end{align*}
As before, a payoff with expectation $u$ is realized as a random variable with distribution $2{\rm Ber}(\frac{1+u}{2})-1$. It can be seen that the only difference between $G_0$ and $G_{j,a}$ lies in  $u_j(a,\{1\}^{N-1})$. By the KL-divergence chain rule, for any algorithm $\mathcal{O}$,
\begin{align*}
    KL\left(\left.\mathcal{O}(G_0) \right\Vert \mathcal{O}(G_{j,a})\right) &\le 10\Delta^2\cdot \E_{G_0}\left[n(a_j=a,a_{-j}=\{1\}^{N-1})\right],
\end{align*}
where $n(a_j=a,a_{-j}=\{1\}^{N-1})$ denotes the number of times the action profile $(a,1^{N-1})$ is played.
Note that in $G_0$, the only action profile surviving two rounds of $\Delta$-IDE is $(1,\cdots,1)$, while in $G_{j,a}$, the only rationalizable action profile is $(\underbrace{1,\cdots,1}_{j-1},a,1,\cdots,1)$. To guarantee $0.9$ accuracy, by Pinsker's inequality,
$${\rm KL}\left(\mathcal{O}(G_0) || \mathcal{O}(G_{j,a})\right) \ge \frac{1}{2}\left|\mathcal{O}(G_0)-\mathcal{O}(G_{j,a})\right|^2>1.$$
It follows that $\forall j\in [N], a>1$,
$$\E_{G_0}\left[n(a_j=a,a_{-j}=\{1\}^{N-1})\right]\ge \frac{1}{10\Delta^2}.$$
Thus the total expected sample complexity is at least
\begin{align*}
    \sum_{a>1,j\in [N]}\E_{G_0}\left[n(a_j=a,a_{-j}=\{1\}^{N-1})\right]\ge \frac{N(A-1)}{10\Delta^2}.
\end{align*}
\end{proof}
\section{Omitted Proofs in Section~\ref{sec:cce}}
\label{sec:cceproof}
We start our analysis by bounding the sampling noise.
For player $i\in [N]$, action $a_i\in\cA_i$,  and $\tau\in[T]$, we denote the sampling noise as
\[
\xi_i^{({\tau})}(a_i):=u_i^{({\tau})}(a_i)-u_i(a_i,\theta_{-i}^{(\tau)}).
\]
We have the following lemma.
\begin{lem}
\label{lemma:cce-event}
Let $\Omega_1$ denote the event that  for all $t\in[T]$, $i\in[N]$, and $a_i\in\cA_i$,
\[
\left|\sum_{\tau=1}^t \xi_i^{({\tau})}(a_i)\right|\leq 2\sqrt{\log(ANT/\delta)\sum_{\tau=1}^t \frac{1}{M_\tau}}.
\]
Then $\Pr[\Omega_1] \ge 1-\delta$.
\end{lem}
\begin{proof}
Note that$\sum_{\tau=1}^t \xi_i^{({\tau})}(a_i)$ can be written as the sum of $\sum_{\tau=1}^t M_\tau$ mean-zero bounded terms. By Azuma-Hoeffding inequality, with probability at least $1-\frac{\delta}{ANT}$, for a fixed $i\in [N]$, $t\in [T]$, $a_i\in \ca_i$,
\begin{equation}
\label{eq:cce-concentration}
    \left|\sum_{\tau=1}^t \xi_i^{({\tau})}(a_i)\right|\leq 2\sqrt{\log(ANT/\delta)\sum_{\tau=1}^t M_\tau\cdot\left(\frac{1}{M_\tau}\right)^2}.
\end{equation}
A union bound over $i\in [N]$, $t\in [T]$, $a_i\in \ca_i$ proves the statement.
\end{proof}
\begin{lem}
\label{lem:cce-adaptive-ide}
With probability at least $1-2\delta$, for all $t\in[T]$ and all $i\in[N]$, $a_i\in\cA_i\cap E_L$,
\[
 \theta_i^{(t)}(a_i)\leq p.
\]
\end{lem}
\begin{proof}
We condition on the event $\Omega_1$ defined in Lemma~\ref{lemma:cce-event} and the success of Algorithm~\ref{alg:findaction}. We prove the claim by induction in $t$. The base case for $t=1$ holds directly by initialization. Now we assume the case for $1,2,\ldots,t$ holds and consider the case of ${t+1}$.

Consider a fixed player $i\in[N]$ and iteratively dominated action $a_i\in\cA_i\cap E_L$. By definition there exists a mixed strategy $x_i$ such that for all $a_{-i}\cap E_L=\emptyset$,
\[
u_i(x_i,a_{-i})\geq u_i(a_i,a_{-i}) + \Delta.
\]
Therefore for $\tau\in[t]$, by the induction hypothesis for $\tau$,
\begin{align}
u_i(x_i,\theta_{-i}^{(\tau)})&\geq u_i(a_i,\theta_{-i}^{(\tau)}) + (1-ANp)\cdot\Delta-ANp \nonumber\\
&\geq u_i(a_i,\theta_{-i}^{(\tau)}) + \Delta/2. \label{eq:gap-cce-induction}
\end{align}
Consequently,
\begin{align*}
    &\sum_{\tau=1}^t (u_i^{({\tau})}(x_i)-u_i^{({\tau})}(a_i))\\
    \geq& \sum_{\tau=1}^t (u_i(x_i,\theta_{-i}^{(\tau)})-u_i(a_i,\theta_{-i}^{({\tau})}))-4\cdot\sqrt{\log(ANT/\delta)\sum_{\tau=1}^t \frac{1}{M_\tau}} \tag{By (\ref{eq:cce-concentration})}\\
    \geq& \frac{t\Delta}{2}-4\cdot\sqrt{\log(ANT/\delta)\sum_{\tau=1}^t \frac{1}{M_\tau}} \tag{By (\ref{eq:gap-cce-induction})}\\
    \geq & \frac{t\Delta}{4}.
\end{align*}
Therefore by our choice of learning rate,
\begin{align*}
{\theta}_{i}^{(t+1)}(a_i) &\le \exp\left(-\eta_{t}\cdot  \sum_{\tau=1}^t \left(u_i^{(\tau)}(x_i)-u_i^{(\tau)}(a_i)\right)\right)\\
&\le \exp\left(-\frac{4\ln(1/p)}{\Delta t}\cdot \frac{\Delta t}{4}\right)= p.
\end{align*}
Therefore
\[
\theta_{i}^{(t+1)}(a_i)\leq p
\]
as desired.
\end{proof}
Now we turn to the $\epsilon$-CCE guarantee. For a player $i\in[N]$, recall that the regret is defined as
\[
{\rm Regret}_T^i=\max_{\theta\in\Delta({\cA_i})}\sum_{t=1}^T \langle  u^{(t)}_i, \theta-\theta_i^{(t)}\rangle.
\]
\begin{lem}
\label{lem:cceregret}
The regret can be bounded as
\[
{\rm Regret}_T^i\leq O\left(\sqrt{\ln A\cdot T} +\frac{\ln(1/p)\ln T}{\Delta}\right).
\]
\end{lem}
\begin{proof}
Note that apart from the choice of $\theta^{(1)}$, we are exactly running FTRL with learning rates
\[
\eta_t=\max\left\{\sqrt{\ln A/t},\frac{4\log(1/p)}{\Delta t}\right\},
\]
which are monotonically decreasing. Therefore following the standard analysis of FTRL (see, \emph{e.g.}, \citet[Corollary 7.9]{orabona2019modern}), we have
\begin{align*}
    \max_{\theta\in\Delta({\cA_i})}\sum_{t=1}^T \langle  u_i^{(t)}, \theta-\theta_i^{(t)}\rangle &\le 2+\frac{\ln A}{\eta_T} + \frac{1}{2}\sum_{t=1}^T\eta_t\\
    &\leq 2+\sqrt{\ln A \cdot T}+ \frac{1}{2}\sum_{t=1}^T \left(\sqrt{\frac{\ln A}{t}}+\frac{4\ln(1/p)}{\Delta t}\right)\\
    &= O\left(\sqrt{\ln A\cdot T} +\frac{\ln(1/p)\ln T}{\Delta}\right).
\end{align*}
\end{proof}
However, this form of regret cannot directly imply approximate CCE. We define the following expected version regret
\[
{\rm Regret}_T^{i,\*}=\max_{\theta\in\Delta({\cA_i})}\sum_{t=1}^T \langle  u_i(\cdot,\theta_{-i}^{(t)}), \theta-\theta_i^{(t)}\rangle.
\]
The next lemma bound the difference between these two types of regret
\begin{lem}
\label{lem:tworegret}
The following event $\Omega_2$ holds with probability at least $1-\delta$: for all $i\in[N]$
\[
\left|{\rm Regret}_T^{i,\*}-{\rm Regret}_T^i\right|\leq O\left(\sqrt{T\cdot \ln(NA/\delta)}\right).
\]

\end{lem}
\begin{proof}
We denote 
\[
\Theta_i:=\{\mathbf{e}_1,\mathbf{e}_2,\ldots,\mathbf{e}_{|\cA_i|}\}
\]
Therefore we have
\begin{align*}
&\left|{\rm Regret}_T^{i,\*}-{\rm Regret}_T^i\right|\\
=& \left|\max_{\theta\in\Delta({\cA_i})}\sum_{t=1}^T \langle  u_i(\cdot,\theta_{-i}^{(t)}), \theta-\theta_i^{(t)}\rangle-\max_{\theta\in\Delta({\cA_i})}\sum_{t=1}^T \langle  u^{(t)}_i, \theta-\theta_i^{(t)}\rangle\right|\\
=& \left|\max_{\theta\in\Theta_i}\sum_{t=1}^T \langle  u_i(\cdot,\theta_{-i}^{(t)}), \theta-\theta_i^{(t)}\rangle-\max_{\theta\in\Theta_i}\sum_{t=1}^T \langle  u^{(t)}_i, \theta-\theta_i^{(t)}\rangle\right|\\
=& \max_{\theta\in\Theta_i}\left|\sum_{t=1}^T \langle  u_i(\cdot,\theta_{-i}^{(t)}), \theta-\theta_i^{(t)}\rangle-\sum_{t=1}^T \langle  u_i^{(t)}, \theta-\theta_i^{(t)}\rangle\right|\\
=& \max_{\theta\in\Theta_i}\left|\sum_{t=1}^T \langle  u_i(\cdot,\theta_{-i}^{(t)})-u_i^{(t)}, \theta-\theta_i^{(t)}\rangle\right|
\end{align*}
Note that $\langle  u_i(\cdot,\theta_{-i}^{(t)})-u_i^{(t)}, \theta-\theta_i^{(t)}\rangle$ is a bounded martingale difference sequence. By Azuma-Hoeffding's inequality, for a fixed $\theta\in\Theta_i$, with probability at least $1-\frac{\delta}{AN}$,
\[
\left|\sum_{t=1}^T \langle  u_i(\cdot,\theta_{-i}^{(t)})-u_i^{(t)}, \theta-\theta_i^{(t)}\rangle\right|\leq O\left(\sqrt{T\cdot \ln(NA/\delta)}\right)
\]
Thus we complete the proof by a union bound.
\end{proof}
\begin{proof}[Proof of Theorem~\ref{thm:cce_main}]
We condition on event $\Omega_1$ defined Lemma~\ref{lemma:cce-event}, event $\Omega_2$ defined in Lemma~\ref{lem:tworegret}, and the success of Algorithm~\ref{alg:findaction}.

\textbf{Coarse Correlated Equilibria.}
By Lemma~\ref{lem:cceregret} and Lemma~\ref{lem:tworegret} we know that for all $i\in[N]$,
\[
{\rm Regret}_T^{i,\*}\leq O\left(\sqrt{\ln A\cdot T} +\frac{\ln(1/p)\ln T}{\Delta}+\sqrt{T\cdot\ln(NA/\delta)}\right).
\]
Therefore choosing
\[
T=\Theta\left(\frac{\ln(NA/\delta)}{\epsilon^2}+\frac{\ln^2(NA/\Delta\epsilon\delta)}{\Delta\epsilon}\right)
\]
will guarantee that ${\rm Regret}_T^{i,\*}$ is at most $\epsilon T/2$ for all $i\in[N]$. In this case the average strategy $(\sum_{t=1}^T \otimes_{i=1}^N {\theta}_i^{(t)})/T$ would be an $(\epsilon/2)$-CCE.

Finally, in the clipping step, $\Vert \bar\theta_i^{(t)}-\theta_i^{(t)} \Vert_1 \le 2pA \le \frac{\epsilon}{4N}$ for all $i\in [N]$, $t\in [T]$. Thus for all $t\in[T]$, we have $\Vert \otimes_{i=1}^n\bar\theta_i^{(t)}-\otimes_{i=1}^n\theta_i^{(t)} \Vert_1\leq \frac{\epsilon}{4}$, which further implies \[\left\Vert (\sum_{t=1}^T\otimes_{i=1}^n\bar\theta_i^{(t)})/T-(\sum_{t=1}^T\otimes_{i=1}^n\theta_i^{(t)})/T \right\Vert_1\leq \frac{\epsilon}{4}.\] Therefore the output strategy $\Pi=(\sum_{t=1}^T \otimes_{i=1}^N \bar{\theta}_i^{(t)})/T$ is an $\epsilon$-CCE.

\textbf{Rationalizability.} By Lemma~\ref{lem:cce-adaptive-ide}, if $a\in E_L\cap \ca_i$, $\theta_i^{(t)}(a)\le p$ for all $t\in [T]$. It follows that $\bar\theta_i^{(t)}(a)=0$, \emph{i.e.}, the action would not be the support in the output strategy $\Pi=(\sum_{t=1}^T \otimes_{i=1}^N \bar{\theta}_i^{(t)})/T$.

\textbf{Sample complexity}. The total number of full-information queries is
\begin{align*}
\sum_{t=1}^T M_t &\le  T+\sum_{t=1}^T \frac{64\log(ANT/\delta)}{\Delta^2 t}\\
&\le  T+\Tilde{O}\left(\frac{1}{\Delta^2}\right)\\
   &= \tilde{O}\left(\frac{1}{\Delta^2}+\frac{1}{\epsilon^2}\right).
\end{align*}
The total sample complexity for CCE learning would then be
\begin{align*}
NA\cdot \sum_{t=1}^T M_t=\Tilde{O}\left(\frac{NA}{\epsilon^2}+\frac{NA}{\Delta^2}\right).
\end{align*}
Finally consider the cost of finding one IDE-surviving action profile ($\tilde{O}\left(\frac{LNA}{\Delta^2}\right)$) and we get the claimed rate.
\end{proof}
\section{Omitted Proofs in Section~\ref{sec:ce}}

\label{sec:ce-proof}

Similar to the CCE case we first bound the sampling noise.
For action $a_i\in\cA_i$, and $\tau\in[T]$, we denote the sampling noise as 
\[
\xi_i^{({\tau})}(a_i):=u_i^{({\tau})}(a_i)-u_i(a_i,\theta_{-i}^{(\tau)}).
\]
In the CE case, we are interested in the weighted sum of noise
$\sum_{\tau=1}^t \xi_i^{({\tau})}(a_i)\theta_i^{(\tau)}(b_i),$
which is bounded in the following lemma.
\begin{lem}
\label{lem:ce-adaptive-ide}
The following event $\Omega_3$ holds with probability at least $1-\delta$: for all $t\in[T]$, $i\in[N]$, and $a_i\in\cA_i$,
\[
\left|\sum_{\tau=1}^t \xi_i^{({\tau})}(a_i)\theta_i^{(\tau)}(b_i)\right|\leq\frac{\Delta}{4}\sum_{\tau=1}^t \theta_i^{(\tau)}(b_i).
\]
\end{lem}
\begin{proof}
Note that $\sum_{\tau=1}^t \xi_i^{({\tau})}(a_i)\theta_i^{(\tau)}(b_i)$ can be written as the sum of $\sum_{\tau=1}^t M_i^{(\tau)}$ mean-zero bounded terms. Precisely, there are $M_i^{(\tau)}$ terms bounded by $\frac{\theta_i^{(\tau)}(b_i)}{M_i^{(\tau)}}$.
By the Azuma-Hoeffding inequality, we have that
with probability at least $1-\frac{\delta}{A^2NT}$,
\begin{align*}
\left|\sum_{\tau=1}^t \xi_i^{({\tau})}(a_i)\theta_i^{(\tau)}(b_i)\right|&\leq 2\cdot\sqrt{\log(ANT/\delta)\sum_{\tau=1}^t M_i^{\tau}\cdot\left(\frac{\theta_i^{(\tau)}(b_i)}{M_i^{(\tau)}}\right)^2}\\
&= 2\cdot\sqrt{\log(ANT/\delta)\sum_{\tau=1}^t \frac{(\theta_i^{(\tau)}(b_i))^2}{M_i^{(\tau)}}}\\
&\leq \frac{\Delta}{4}\cdot\sqrt{\sum_{\tau=1}^t \theta_i^{(\tau)}(b_i)\sum_{j=1}^\tau\theta_i^{(j)}(b_i)}\\
&\leq \frac{\Delta}{4}\sum_{\tau=1}^t \theta_i^{(\tau)}(b_i)
\end{align*}
Therefore by a union bound we complete the proof.
\end{proof}
\begin{lemma}
\label{lem:ce-adaptive-ide-real}
With probability at least $1-2\delta$, for all $t\in[T]$, all $i\in[N]$, and all $a_i\in\cA_i\cap E_L$,
\[
{\theta}_i^{(t)}(a_i)\leq p
\]
\end{lemma}
\begin{proof}
We condition on the event $\Omega_3$ defined in Lemma~\ref{lem:ce-adaptive-ide} and the success of Algorithm~\ref{alg:findaction}. We prove the claim by induction in $t$. The base case for $t=1$ holds directly by initialization. Now we assume the case for $1,2,\ldots,t$ holds and consider the case of ${t+1}$.

Consider a fixed player $i\in[N]$, an iteratively dominated action $a_i\in\cA_i\cap E_L$, and an expert $b_i$. By definition there exists a mixed strategy $x_i$ such that for all $a_{-i}\cap E_L=\emptyset$,
\[
u_i(x_i,a_{-i})\geq u_i(a_i,a_{-i}) + \Delta
\]
Therefore for $\tau\in[t]$, by induction hypothesis we have
\begin{align*}
u_i(x_i,\theta_{-i}^{(\tau)})&\geq u_i(a_i,\theta_{-i}^{(\tau)}) + (1-ANp)\cdot\Delta-ANp\\
&\geq u_i(a_i,\theta_{-i}^{(\tau)}) + \Delta/2
\end{align*}
Thus we have
\begin{align*}
    &\sum_{\tau=1}^t (u_i^{({\tau})}(x_i)-u_i^{({\tau})}(a_i))\cdot\theta_i^{(\tau)}(b_i)\\
    \geq& \sum_{\tau=1}^t (u_i(x_i,\theta_{-i}^{(\tau)})-u_i(a_i,\theta_{-i}^{({\tau})}))\cdot\theta_i^{(\tau)}(b_i)-\frac{\Delta}{4}\sum_{\tau=1}^t \theta_i^{(\tau)}(b_i)\\
    \geq& \frac{\Delta}{2}\sum_{\tau=1}^t \theta_i^{(\tau)}(b_i)-\frac{\Delta}{4}\sum_{\tau=1}^t \theta_i^{(\tau)}(b_i)\\
    = & \frac{\Delta}{4}\sum_{\tau=1}^t \theta_i^{(\tau)}(b_i)
\end{align*}
By our choice of learning rate,
\begin{align*}
\hat{\theta}_{i}^{(t+1)}(a_i|b_i) &\le \exp\left(-\eta_{t,i}^{b_i}\cdot  \sum_{\tau=1}^t \theta_i^{(\tau)}(b_i)\left(u_i^{(\tau)}(x_i)-u_i^{(\tau)}(a_i)\right)\right)\\
&\le \exp\left(-\frac{4\ln(1/p)}{\Delta\sum_{\tau=1}^t \theta_i^{(\tau)}(b)}\cdot \frac{\Delta}{4}\sum_{i=1}^t \theta_i^{(\tau)}(b)\right)= p.
\end{align*}
Therefore we conclude
\[
\theta_i^{(t+1)}(a_i)=\sum_{b_i\in\cA_i}\hat\theta^{(t+1)}_i(a_i|b_i)\theta^{(t+1)}_i(b_i)\leq p
\]
\end{proof}
Now we turn to the $\epsilon$-CE guarantee. For a player $i\in[N]$, recall that the swap-regret is defined as
\begin{align*}
    {\rm SwapRegret}_T^i := \sup_{\phi:\ca_i\to\ca_i}\sum_{t=1}^T\sum_{b\in\ca_i} \theta_i^{(t)}(b)u_i^{(t)}(\phi(b)) - \sum_{t=1}^T \left\langle \theta_i^{(t)},u_i^{(t)} \right\rangle.
\end{align*}
\begin{lemma}
\label{lem:ceregret}
For all $i\in[N]$, the swap-regret can be bounded as
\[
{\rm SwapRegret}_T^{i}\le  O\left(\sqrt{A\ln(A)T} + \frac{A\ln(NAT/\Delta\epsilon)^2}{\Delta}\right).
\]
\end{lemma}


\begin{proof}
For $i\in[N]$, recall that the regret for an expert $b\in\cA_i$ is defined as
\begin{align*}
    {\rm Regret}_T^{i,b} := \max_{a\in \ca_i}\sum_{t=1}^T \theta_i^{(t)}(b)u_i^{(t)}(a)  - \sum_{t=1}^T \left\langle \hat\theta_i^{(t)}(\cdot|b),\theta_i^{(t)}(b)u_i^{(t)} \right\rangle.
\end{align*}
Since $\theta_i^{(t)}(a) = \sum_{b\in\ca_i} \hat\theta_i^{(t)}(a|b)\theta_i^{(t)}(b)$ for all $a$ and all $t>1$,
\begin{align*}
    \sum_{b\in\ca_i}{\rm Regret}_T^{i,b} &= \sum_{b\in\ca_i} \max_{a_b\in \ca_i}\sum_{t=1}^T \theta_i^{(t)}(b)u_i^{(t)}(a_b)  - \sum_{b\in\ca_i}\sum_{t=1}^T \left\langle \hat\theta_i^{(t)}(\cdot|b)\theta_i^{(t)}(b),u_i^{(t)} \right\rangle\\
    &=\max_{\phi:\ca_i\to\ca_i}\sum_{b\in\ca_i}\sum_{t=1}^T\theta_i^{(t)}(b)u_i^{(t)}(\phi(b)) - \sum_{t=1}^T \left\langle \sum_{b\in\ca_i}\hat\theta_i^{(t)}(\cdot|b)\theta_i^{(t)}(b), u_i^{(t)}\right\rangle\\
    &\ge  \max_{\phi:\ca_i\to\ca_i}\sum_{t=1}^T\sum_{b\in\ca_i}\theta_i^{(t)}(b)u_i^{(t)}(\phi(b))
    - \sum_{t=2}^T \left\langle \theta_i^{(t)},u_i^{(t)} \right\rangle - 1\ge  {\rm SwapRegret}_T^i-1.
\end{align*}
It now suffices to control the regret of each individual expert. For expert $b$, we are essentially running FTRL with learning rates
$$\eta_{t,i}^b:=\max\left\{\frac{4\ln(1/p)}{\Delta\sum_{\tau=1}^{t}\theta_i^{(\tau)}(b)},\frac{\sqrt{A\ln A}}{\sqrt{t}}\right\},$$
which are clearly monotonically decreasing. Therefore using standard analysis of FTRL (see, \emph{e.g.}, \citet[Corollary 7.9]{orabona2019modern}),
\begin{align*}
{\rm Regret}_T^{i,b} &\le \frac{\ln A}{\eta_{T,i}^b} + \sum_{t=1}^T \eta_{t,i}^b \cdot \theta_i^{(t)}(b)^2\\
&\le \sqrt{\frac{T\ln A}{A}} + \sum_{t=1}^T\theta_i^{(t)}(b)\cdot \sqrt{\frac{A\ln A}{t}} + \frac{4\ln(1/p)}{\Delta}\cdot\sum_{t=1}^{T} \frac{\theta_i^{(t)}(b)}{\sum_{\tau=1}^t \theta_i^{(\tau)}(b)}\\
&\le \sqrt{\frac{T\ln A}{A}}  +  \sum_{t=1}^T\theta_i^{(t)}(b)\cdot \sqrt{\frac{A\ln A}{t}} + \frac{4\ln(1/p)}{\Delta}\left(1+\ln\left(\frac{T}{p}\right)\right).
\end{align*}
Here we used the fact that $\forall b\in \ca_i$, $\theta_i^{(1)}(b)\ge p$, and
\begin{align*}
    \sum_{t=1}^{T} \frac{\theta^{(t)}_i(b)}{\sum_{i=1}^\tau \theta_i^{(\tau)}(b)} &\le 1+\int_{\theta_i^{(1)}(b)}^{\sum_{t=1}^T \theta_i^{(t)}(b)}{\frac{{\rm d}s}{s}} = 1 + \ln\left(\frac{\sum_{t=1}^T \theta_i^{(t)}(b)}{\theta_i^{(1)}(b)}\right)\\
    &\le 1 + \ln\left(\frac{T}{p}\right).
\end{align*}
Notice that $\sum_{b\in\ca_i}\sum_{t=1}^T\theta_i^{(t)}(b)\cdot \sqrt{\frac{A\ln A}{t}} \le O(\sqrt{A\ln (A) T})$. Therefore
\begin{equation}
\label{eq:swap-regret-ce}
    {\rm SwapRegret}_T^i\le O(1)+ \sum_{b\in \ca_i} {\rm Regret}_T^{i,b} \le O\left(\sqrt{A\ln(A)T} + \frac{A\ln(NAT/\Delta\epsilon)^2}{\Delta}\right).
\end{equation}
\end{proof}
Similar to the CCE case,, this form of regret can not directly imply approximate CE. We define the following expected version regret
\[
{\rm SwapRegret}_T^{i,\*}:=\sup_{\phi:\ca_i\to\ca_i}\sum_{t=1}^T\left\langle \phi\circ \theta_i^{(t)},u_i(\cdot,\theta_{-i}^{(t)})\right\rangle - \sum_{t=1}^T \left\langle \theta_i^{(t)},u_i(\cdot,\theta_{-i}^{(t)}) \right\rangle
\]
The next lemma bound the difference between these two types of regret
\begin{lem}
\label{lem:cetworegret}
The following event $\Omega_4$ has probability at least $1-\delta$: for all $i\in[N]$,
\[
\left|{\rm SwapRegret}_T^{i,\*}-{\rm SwapRegret}_T^i\right|\leq O\left(\sqrt{AT\ln\left(\frac{AN}{\delta}\right)}\right).
\]
\end{lem}
\begin{proof}
Note that
\begin{align*}
&\left|{\rm SwapRegret}_T^{i,\*}-{\rm SwapRegret}_T^i\right|\\
=&\left|\sup_{\phi:\ca_i\to\ca_i}\sum_{t=1}^T\left\langle \phi\circ \theta_i^{(t)}-\theta_i^{(t)},u_i(\cdot,\theta_{-i}^{(t)})\right\rangle-\sup_{\phi:\ca_i\to\ca_i}\sum_{t=1}^T\left\langle \phi\circ \theta_i^{(t)}-\theta_i^{(t)},u_i^{(t)}\right\rangle\right|\\
\leq&\sup_{\phi:\ca_i\to\ca_i}\left|\sum_{t=1}^T\left\langle \phi\circ \theta_i^{(t)}-\theta_i^{(t)},u_i(\cdot,\theta_{-i}^{(t)})-u_i^{(t)}\right\rangle\right|.
\end{align*}
Notice that $\E[u_i^{(t)}] = u_i\left(\cdot,\theta_{-i}^{(t)}\right)$, and that $u_i^{(t)}\in [-1,1]^A$. Therefore, $\forall \phi:\ca_i\to \ca_i$,  $$\xi^\phi_t:=\left\langle \phi\circ \theta_i^{(t)}-\theta_i^{(t)},u_i\left(\cdot,\theta_{-i}^{(t)}\right)-u_i^{(t)}\right\rangle$$ is a bounded martingale difference sequence. By Azuma-Hoeffding inequality, for a fixed $\phi:\cA_i\rightarrow\cA_i$, with probability $1-\delta'$,
\begin{align*}
    \left|\sum_{t=1}^T \xi_t^\phi\right| \le 2\sqrt{2T\ln\left(\frac{2}{\delta'}\right)}.
\end{align*}
By setting $\delta'=\delta/(NA^A)$, we get with probability $1-\delta/N$, $\forall \phi:\ca_i\to \ca_i$,
\begin{align*}
  \left|\sum_{t=1}^T \xi_t^\phi\right| \le 2\sqrt{2AT\ln\left(\frac{2AN}{\delta}\right)}.
\end{align*}
Therefore we complete the proof by a union bound over $i\in[N]$.
\end{proof}

\begin{proof}[Proof of Theorem~\ref{thm:ce_main}]
We condition on event $\Omega_3$ defined Lemma~\ref{lem:ce-adaptive-ide}, event $\Omega_4$ defined in Lemma~\ref{lem:cetworegret}, and the success of Algorithm~\ref{alg:findaction}.

\textbf{Correlated Equilibrium.} 
By Lemma~\ref{lem:ceregret} and Lemma~\ref{lem:cetworegret} we know that for all $i\in[N]$,
\[
{\rm SwapRegret}_T^{i,\*}\leq O\left(\sqrt{A\ln(A)T} + \frac{A\ln(NAT/\Delta\epsilon)^2}{\Delta}+\sqrt{AT\ln\left(\frac{AN}{\delta}\right)}\right).
\]
Therefore choosing 
\begin{align*}
    T=\Theta\left(\frac{A\ln\left(\frac{AN}{\delta}\right)}{\epsilon^2} + \frac{A\ln^3\left(\frac{NA}{\Delta\epsilon\delta}\right)}{\Delta\epsilon}\right)
\end{align*}
will guarantee that ${\rm SwapRegret}_T^{i,\*}$ is at most $\epsilon T/2$ for all $i\in[N]$. In this case the average strategy $(\sum_{t=1}^T \otimes_{i=1}^N {\theta}_i^{(t)})/T$ would be an $\epsilon/2$-CE.

Finally, in the clipping step, $\Vert \bar\theta_i^{(t)}-\theta_i^{(t)} \Vert_1 \le 2pA \le \frac{\epsilon}{4N}$ for all $i\in [N]$, $t\in [T]$. Thus for all $t\in[T]$, we have $\Vert \otimes_{i=1}^n\bar\theta_i^{(t)}-\otimes_{i=1}^n\theta_i^{(t)} \Vert_1\leq \frac{\epsilon}{4}$, which further implies \[\left\Vert (\sum_{t=1}^T\otimes_{i=1}^n\bar\theta_i^{(t)})/T-(\sum_{t=1}^T\otimes_{i=1}^n\theta_i^{(t)})/T \right\Vert_1\leq \frac{\epsilon}{4}.\] Therefore the output strategy $\Pi=(\sum_{t=1}^T \otimes_{i=1}^N \bar{\theta}_i^{(t)})/T$ is an $\epsilon$-CE.


\textbf{Rationalizability.} By Lemma~\ref{lem:ce-adaptive-ide-real}, if $a\in E_L\cap \ca_i$, $\theta_i^{(t)}(a)\le p$ for all $t\in [T]$. It follows that $\bar\theta_i^{(t)}(a)=0$, \emph{i.e.}, the action would not be the support in the output strategy $\Pi=(\sum_t \otimes_i \bar{\theta}_i^{(t)})/T$.

\textbf{Sample complexity}. The total number of queries is
\begin{align*}
\sum_{i\in[N]}\sum_{t=1}^T AM_i^{(t)} &\le NAT+ \sum_{i\in[N]}\sum_{b\in\ca_i} \sum_{t=1}^{T} \frac{16\theta_i^{(t)}(b)}{\Delta^2 \cdot \sum_{\tau=1}^t \theta_i^{(\tau)}(b)}\\
&\le NAT+\frac{16 NA^2}{\Delta^2}\cdot \ln (T/p)\\
&\le \Tilde{O}\left(\frac{NA^2}{\epsilon^2}+\frac{NA^2}{\Delta^2}\right),
\end{align*}
where we used the fact that
\begin{align*}
    \sum_{t=1}^{T} \frac{\theta^{(t)}_i(a)}{\sum_{\tau=1}^t \theta_i^{(\tau)}(a)}\le 1 + \ln\left(\frac{T}{p}\right).
\end{align*}
Finally consider the cost of finding one IDE-surviving action profile ($\tilde{O}\left(\frac{LNA}{\Delta^2}\right)$) and we get the claimed rate.
\end{proof}
\section{Omitted Proofs in Section~\ref{sec:reduction}}
\label{sec:reductiondetails}


\subsection{Proof of Theorem~\ref{thm:cce-reduction}}
In the proof we denote $\epsilon'=\frac{\min\{\epsilon,\Delta\}}{3}$.

\begin{lemma}
\label{lem:cce-reduction-concentration}
With probability $1-\delta$, throughout the execution of Algorithm~\ref{alg:cce-reduction}, for every $t$ and $i\in [N]$, $a'_i\in \cA_i$,
\begin{align*}
    \left|\hat u_i(a'_i, \Pi_{-i}) - u_i\left(a'_i, \Pi_{-i}\right)\right| \le \epsilon'.
\end{align*}
\end{lemma}
\begin{proof}
First, observe that during every iterate of $t$ before the algorithm returns, the total support size $\sum_{i=1}^{t} |\ca_i^{(t)}|$ is increased by at least $1$. It follows that the algorithm returns before $t=NA$.

By Hoeffding's inequality,
\begin{align*}
\Pr\left[ \left|\hat u_i(a'_i, \Pi_{-i}) - u_i\left(a'_i, \Pi_{-i}\right)\right|>\epsilon'\right] \le 2\exp\left(-\frac{n\epsilon'^2}{2}\right) \le \frac{\delta}{N^2A^2}.
\end{align*}
Applying union bound over $t$, $i$ and $a'_i$ proves the statement.
\end{proof}

\begin{proof}[Proof of Theorem~\ref{thm:cce-reduction}]
\textbf{Correctness.} Since $\Pi$ is an $\epsilon$-CCE in the subgame $\Pi_{i=1}^{N} \ca_i^{(t)}$, $\forall i\in [N]$, $\forall a\in \ca_i^{(t)}$
\begin{align*}
u_i\left(a,\Pi_{-i}\right) \le u_i\left(\Pi\right) + \epsilon'. 
\end{align*}
Because $\argmax_{a\in \ca_i} \hat u_i(a,\Pi_{-i}) \in \ca_i^{(t)}$,
$\forall i\in [N]$, $\forall a\in \ca_i$
\begin{align*}
u_i\left(a,\Pi_{-i}\right) &\le \hat u_i\left(a,\Pi_{-i}\right)+\epsilon' \le \max_{a'\in \ca_i^{(t)}}\hat u_i\left(a',\Pi_{-i}\right)+\epsilon' \\
&\le \max_{a'\in \ca_i^{(t)}}u_i\left(a',\Pi_{-i}\right)+2\epsilon' \\
&\le u_i(\Pi) + 3\epsilon' \le u_i(\Pi) + \epsilon.
\end{align*}
Therefore $\Pi$ is an $\epsilon$-CCE in the full game.

Moreover, we claim that for any $t$, $\ca_i^{(t)}$ only contains  $\Delta$-rationalizable actions. This is true for $t=1$ with high probability due to our initialization. Suppose that this is true for $t$. Notice that the only way for an action $a'_i\in \ca_i^{(t+1)}$ is to be an empirical best response, which means
\begin{align*}
    u_i(a'_i , \Pi_{-i}) &\ge \hat u_i(a'_i,\Pi_{-i}) - \epsilon' \ge  \max_{a\in \ca_i} \hat u_i(a, \Pi_{-i}) - \epsilon'\\
    &\ge \max_{a\in \ca_i} u_i(a, \Pi_{-i}) - 2\epsilon'.
\end{align*}
Since $\epsilon'<\Delta/2$, this means that $a'_i$ is the $\Delta$-best response to a  $\Delta$-rationalizable strategy, and is therefore  $\Delta$-rationalizable. Therefore $\ca_i^{(t+1)}$ also only contains $\Delta$-rationalizable actions. Our claim can be thus proven via induction, and it follows that the output strategy is also $\Delta$-rationalizable. 

We conclude that the output strategy is a $\Delta$-rationalizable $\epsilon$-CCE with probability $1-2\delta$ (assuming the event in Lemma~\ref{lem:cce-reduction-concentration} as well as the rationalizability of the initialization).

\textbf{Sample complexity.} By Theorem~\ref{thm:ide_upper}, Line $1$ needs $\tilde{O}\left(\frac{LNA}{\Delta^2}\right)$ samples. Since the algorithm returns before $t=NA$, the total number of calls to the black-box oracle $\mathcal{O}$ is $NA$. For each $t$, the number of samples required is
\begin{align*}
    NAM = \tilde{O}\left(\frac{NA}{\min\{\Delta,\epsilon\}^2}\right).
\end{align*}
Combining this with the upper bound on $t$, and the cost for Algorithm~\ref{alg:findaction} gives the total sample complexity bound
\begin{align*}
    \Tilde{O}\left(\frac{N^2A^2}{\min\{\Delta^2,\epsilon^2\}}\right).
\end{align*}
\end{proof}

\subsection{Proof of Theorem~\ref{thm:ce-reduction}}
In the proof we denote $\epsilon'=\frac{\min\{\epsilon,\Delta\}}{3}$.

\begin{lemma}
\label{lem:ce-reduction-concentration}
With probability $1-\delta$, throughout the execution of Algorithm~\ref{alg:cce-reduction}, for every $t$ and $i\in [N]$, $a'_i\in \cA_i$, $a_i\in \cA_i$,
\begin{align*}
    \left|\hat u_i(a'_i, \Pi|a_i) - u_i\left(a'_i, \Pi|a_i\right)\right| \le \epsilon'.
\end{align*}
\end{lemma}
\begin{proof}
First, observe that during every iterate of $t$ before the algorithm returns, the total support size $\sum_{i=1}^{t} |\ca_i^{(t)}|$ is increased by at least $1$. It follows that the algorithm returns before $t=NA$.

By Hoeffding's inequality,
\begin{align*}
\Pr\left[ \left|\hat u_i(a'_i, \Pi|a_i) - u_i\left(a'_i,  \Pi|a_i\right)\right|>\epsilon'\right] \le 2\exp\left(-\frac{n\epsilon'^2}{2}\right) \le \frac{\delta}{N^2A^3}.
\end{align*}
Applying union bound over $t$, $i$, $a_i$, and $a'_i$ proves the statement.
\end{proof}

\begin{proof}
Note that with high probability, the empirical estimates $\hat U$ are at most $\epsilon/4$ away from the true value $U$. Since $a_i'$ is the empirical best response, we have
\begin{align*}
    	U_i(a'_i,\Pi|a_i) \geq \argmax_{a\in\cA_i}U_i(a,\Pi|a_i) - \epsilon .
\end{align*}
Note that $\Pi|a_i$ is supported on actions that can survive any rounds of $\epsilon$-IDE. Therefore it serves as a certificate that $a_i'$ will never be $\epsilon$-eliminated as well.
\end{proof}

\begin{lemma}
\label{lem:ce-reduction-correct}
The returned strategy $\Pi$ is an $\epsilon$-CE with probability $1-\delta$.
\end{lemma}
\begin{proof}
When the algorithm terminates, for all $i\in[N]$, 
\begin{align*}
    \sum_{a_i\in {\cA}^{(t)}_i} \Pi_i(a_i)\cdot \left(\max_{a\in\cA_i} \hat u_i(a,\Pi|a_i)-\max_{a\in\cA_i^{(t)}} \hat u_i(a,\Pi|a_i)\right) = 0.
\end{align*}
Therefore
\begin{align*}
    \sum_{a_i\in {\cA}^{(t)}_i} \Pi_i(a_i)\cdot \left(\max_{a\in\cA_i}  u_i(a,\Pi|a_i)-\max_{a\in\cA_i^{(t)}} u_i(a,\Pi|a_i)\right) \le  2\epsilon'.
\end{align*}
Since $\Pi$ is an $\epsilon'$-CE in the reduced game,
\begin{align*}
    \sum_{a_i\in {\cA}^{(t)}_i} \Pi_i(a_i)\cdot \left(\max_{a\in \cA^{(t)}_i}  u_i(a,\Pi|a_i)-  u_i(a_i,\Pi|a_i)\right)\le \epsilon'.
\end{align*}
Summing the two inequalities above gives
\begin{align*}
    \sum_{a_i\in {\cA}^{(t)}_i} \Pi_i(a_i)\cdot \left(\max_{a\in \cA_i}  u_i(a,\Pi|a_i)-  u_i(a_i,\Pi|a_i)\right)\le 3\epsilon',
\end{align*}
which proves the statement.
\end{proof}

\begin{lemma}
\label{lem:ce-reduction-ide}
For any $t$, $\ca_i^{(t)}$ only contains $\Delta$-rationalizable actions with probability $1-2\delta$.
\end{lemma}
\begin{proof}
We prove this inductively. This is true for $t=1$ with probability $1-\delta$ due to our initialization. Suppose that this is true for $t$. Notice that the only way for an action $a'_i\in \ca_i^{(t+1)}$ is to be an empirical best response, which means for some $a_i$
\begin{align*}
    u_i(a'_i , \Pi|a_i) &\ge \hat u_i(a'_i,\Pi|a_i) - \epsilon' \ge  \max_{a\in \ca_i} \hat u_i(a, \Pi|a_i) - \epsilon'\\
    &\ge \max_{a\in \ca_i} u_i(a, \Pi|a_i) - 2\epsilon'.
\end{align*}
Since $\epsilon'<\Delta/2$, this means that $a'_i$ is the $\Delta$-best response to a  $\Delta$-rationalizable strategy, and is therefore  $\Delta$-rationalizable. Therefore $\ca_i^{(t+1)}$ also only contains $\Delta$-rationalizable actions. Our claim can be thus proven via induction, and it follows that the output strategy is also $\Delta$-rationalizable. 
\end{proof}

\begin{proof}[Proof of Theorem~\ref{thm:ce-reduction}]
\textbf{Correctness.} By Lemma~\ref{lem:ce-reduction-correct} and~\ref{lem:ce-reduction-ide}, the output strategy is a $\Delta$-rationalizable $\epsilon$-CE with probability $1-2\delta$ (assuming that the event in Lemma~\ref{lem:ce-reduction-concentration} holds and the rationalizability of the initialization).

\textbf{Sample complexity.} The total sample complexity is
\begin{align*}
     \Tilde{O}\left(\frac{LNA}{\Delta^2}\right) + NA\times NA^2M = \Tilde{O}\left(\frac{N^2A^3}{\min\{\Delta,\epsilon\}^2}\right).
\end{align*}

\end{proof}
\end{document}